\documentclass[12pt]{article}

\usepackage[margin=1in]{geometry} 
\usepackage{amsmath,amsthm,amssymb}
\usepackage{graphicx}
\usepackage{bbm} 
\usepackage{color} 
\usepackage{multirow}

\usepackage[utf8]{inputenc}
\usepackage[english]{babel}

\usepackage{chngcntr}
 
\newtheorem{theorem}{Theorem}
\newtheorem{corollary}{Corollary}[theorem]
\newtheorem{lemma}[theorem]{Lemma}
\theoremstyle{definition}
\newtheorem{definition}{Definition}[section]
\theoremstyle{assumption}
\newtheorem{assumption}{Assumption}[section]
\theoremstyle{remark}

\newcommand{\norm}[1]{\left\lVert#1\right\rVert}
\newcommand{\abs}[1]{|#1|}
\newcommand{\R}{\mathbb{R}}
\newcommand{\E}{\mathbb{E}}
\newcommand{\X}{\mathbf{X}}
\newcommand{\y}{\mathbf{y}}

\newcommand{\VIC}{\textit{VIC}}

\newcommand{\mr}{\textit{mr}}
\newcommand{\Var}{\textit{Var}}
\newcommand{\Cov}{\textit{Cov}}

\newcommand{\F}{\mathcal{F}}
\newcommand{\Rash}{\mathcal{R}}

\newcommand{\MRel}{\textit{MR}}

\numberwithin{equation}{section}

\usepackage{natbib}
 \bibpunct[, ]{(}{)}{,}{a}{}{,}%

\title{Variable Importance Clouds: A Way to Explore Variable Importance for the Set of Good Models}

\begin{document}

\begin{titlepage}
\title{Variable Importance Clouds: A Way to Explore Variable Importance for the Set of Good Models}
\author{Jiayun Dong \thanks{Department of Economics, Duke University, Durham, NC 27708.} \and Cynthia Rudin \thanks{Departments of Computer Science, Electrical and Computer Engineering, and Statistical Science, Duke University, NC 27708.}}
\date{\today}
\maketitle

\begin{abstract}
\noindent  Variable importance is central to scientific studies, including the social sciences and causal inference, healthcare, and other domains. However, current notions of variable importance are often tied to a specific predictive model. This is problematic: what if there were multiple well-performing predictive models, and a specific variable is important to some of them and not to others? In that case, we may not be able to tell from a single well-performing model whether a variable is always important in predicting the outcome. Rather than depending on variable importance for a single predictive model, we would like to explore variable importance for all approximately-equally-accurate predictive models. This work introduces the concept of a \textit{variable importance cloud}, which maps every variable to its importance for every good predictive model. We show properties of the variable importance cloud and draw connections to other areas of statistics. We introduce \textit{variable importance diagrams} as a projection of the variable importance cloud into two dimensions for visualization purposes. Experiments with criminal justice, marketing data, and image classification tasks illustrate how variables can change dramatically in importance for approximately-equally-accurate predictive models.\\
\vspace{0in}

\bigskip
\end{abstract}
\setcounter{page}{0}
\thispagestyle{empty}
\end{titlepage}

\section{Introduction}
In predictive modeling, how do we know whether a feature is actually important? If we find an accurate predictive model that depends heavily on a feature, it does not necessarily mean that the feature is always important for good models. On the contrary, what if there is another equally accurate model that does not depend on the feature at all? Perhaps in order to answer this question, we need a holistic view of variable importance, that includes not just the importance of a variable to a single model, but to any accurate model. {\em Variable importance clouds}, which we introduce in this work, aims to provide a lens into the secret life of the class of almost-equally-accurate predictive models.

Ideally we would like to obtain a more complete understanding of variable importance for the set of models that predict almost equally well. This set of almost-equally-accurate predictive models is called the {\em Rashomon set}; it is the set of models with training loss below a threshold. The term Rashomon set comes from Breiman's {\em Rashomon effect} \cite{Breimanetal2001}, which is the notion that there could be many good explanations for any given phenomenon. \cite{Breiman2001} also defined a useful notion of {\em variable importance}; namely the increase in loss that occurs when a variable is purposely scrambled (randomly permuted). Unfortunately, however, there is something fundamentally incomplete about considering these two quantities separately: if we look at variable importance only for a single model, we miss the potentially more important question of what the variable importance could be for another different but equally-accurate model. A variable importance cloud (VIC) is precisely the joint set of variable importance values for all models in the Rashomon set. 

Specifically, we define a vector for a single predictive model, each element representing the dependence of the model on a feature. The VIC is the set of such vectors for all models in the Rashomon set. The VIC thus reveals the importance of a feature in the context of the importance of other features for all good models. For example, it may reveal that a feature is important only when another feature is not important, which may happen when these features are highly correlated. Understanding the VIC helps interpret predictive models and provides a context for model selection. This type of analysis provides a deeper understanding of variable importance, going beyond single models and now encompassing the set of every good model. In this paper, we analyze the VIC for linear models, and extend the analysis to some of the nonlinear problems including logistic regression, decision trees, and deep learning.

When there are many features that could be potentially important within at least one good predictive model, the VIC becomes a subset of a high dimensional space. To facilitate understanding of the VIC, we propose a visualization tool called the {\em variable importance diagram} (VID). It is a collection of 2d projections of the VIC onto the space spanned by the importance of a pair of features. The VID offers graphical information about the magnitude of variable importance measures, the bounds, and the relation of variable importance for each pair of features. An upward-sloping projection suggests that a feature is importance only when the other feature is also important, and vice versa for a downward-sloping projection. We provide examples of VIDs in the context of concrete applications, and illustrate how the VID facilitates model interpretation and selection.

The remainder of the paper is organized as follows. In Section 2, we introduce definitions and use linear model as an example to build basic understanding of the VIC/VID. In Section 3, we introduce the general approach of VIC/VID analysis for nonlinear problems, including logistic regression models and decision trees. In Section 4, we describe the use cases of the VIC/VID framework. We demonstrate our framework with concrete examples in Section 5, which includes three experiments. In Section 5.1, we study the Propublica dataset for criminal recidivism prediction and demonstrate the VIC/VID analysis for both logistic regression and decision trees. We move onto an in-vehicle coupon recommendation dataset and illustrate the trade-off between accuracy and variable importance in Section 5.2. We study an image classification problem based on VGG16 in Section 5.3.  We discuss related work in Section 6. As far as we know, there is no other work that aims to visualize the set of variable importance values for the \textit{set} of approximately-equally-good predictive models for a given problem. Instead, past work has mainly defined variable importance for \textit{single} predictive models.

\section{Preliminaries}

For a vector $v \in \R^p$, we denote its $j^{th}$ element by $v_j$ and all elements except for the $j^{th}$ one by $v_{\backslash j}$. For a matrix $M$, we denote its transpose by $M^T$, $i^{th}$ row by $M_{[i,\cdot]}$, and  $j^{th}$ column by $M_{[\cdot,j]}$.

Let $(X, Y) \in \R^{p+1}$ be a random vector of length $p+1$, with $p$ being a positive integer, where $X$ is the vector of $p$ covariate variables (referred to as features) and $Y$ is the outcome variable. Our dataset is an $n-by-(p+1)$ matrix, $(\X, \y) = (x_i, y_i)_{i=1}^n$, where each row $(x_i, y_i)$ is an i.i.d. realization of the random vector $(X, Y)$.

Let $f: \R^p \rightarrow \R$ be a predictive model, and $\F \subset \{f | f: \R^p \rightarrow \R\}$ be the class of predictive models we consider. For a given model $f \in \F$ and an observation $(x, y) \in \R^{p+1}$, let $l(f; x, y)$ be the loss function. The expected loss and empirical loss of model $f$ are defined by $L^{exp}(f; X, Y) = \E[l(f; X,Y)]$ and $L^{emp}(f; (\X, \y)) = \sum_{i=1}^n l(f; x_i, y_i)$. We sometimes drop the superscript or the reliance on the data when the context is clear. We consider different classes of predictive models and loss functions in the paper, including the squared loss, logistic loss, and 0-1 loss.

\subsection{Rashomon Set}
Fix a predictive model $f^* \in \F$ as a benchmark. A model $f \in \F$ is ``good'' if its loss does not exceed the loss of $f^*$ by a factor $\epsilon > 0$. A Rashomon set $\Rash \subset \F$ is defined to be the set of all good models in the class $\F$. In most cases, we select $f^*$ to be the best model within the set $\F$ that minimizes the loss, and we define $f^*$ this way in what follows.

\begin{definition}[Rashomon Set]\label{Rashomon}
Given a model class $\F$, a benchmark model $f^*\in \F$, and $\epsilon > 0$, the {\em Rashomon set} is defined as $$\Rash(\epsilon, f^* ,\F) = \{f \in \F | L(f) \leq (1 + \epsilon) L(f^*) \}.$$
\end{definition}
Note that the Rashomon set $\Rash(\epsilon, f^* ,\F)$ also implicitly depends on the loss function and the dataset.

\subsection{Model Reliance}
For a given model $f \in \F$, we want to measure the degree to which its predictive power relies on a particular variable $j$, where $j = 1, \cdots, p.$ We will use a measure of variable importance that is similar to that used by random forest (\cite{Breiman2001}, see also \cite{FRD2018} for terminology).
Let $(\bar{X}, \bar{Y})$ be another random vector that is independent of and identically distributed to $(X, Y)$. We replace the $X_j$ with $\bar{X}_j$, which gives us a new vector denoted by $([X_{\backslash j}, \bar{X}_j], Y)$.

Intuitively, $L(f; [X_{\backslash j}, \bar{X}_j], Y)$ should be larger than $L(f; X, Y)$, since we have broken the correlation between feature $X_j$ and outcome $Y$. The change in loss due to replacing feature $j$ with a new random draw for feature $j$ is called model reliance. Formally:

\begin{definition}[Model Reliance]\label{MR_def}
The (population) \emph{reliance of model $f$ on variable $j$} is given by either the ratio $$\mr^{\text{ratio}}_j(f) = \frac{L(f; [X_{\backslash j}, \bar{X}_j], Y)}{L(f; X, Y)},$$ or the difference $\mr^{\text{diff}}_j(f) = L(f; [X_{\backslash j}, \bar{X}_j], Y) - L(f; X, Y),$ depending on the specific application. 
\end{definition}

Empirical versions of these quantities can be defined with respect to the empirical dataset and loss function. Larger $\mr_j$ indicates greater reliance on feature $X_j$. We sometimes drop the superscript when the context is clear. 

From here, we diverge from existing work that considers only variable importance of a single function. Let us now define the \textit{model reliance function}, which specifies the importance of each feature to a predictive model.

\begin{definition}[Model Reliance Function]\label{MR}
The function $\MRel: \F \rightarrow \R^p$ maps a model to a vector of its reliance on all features: 
$$\MRel(f) = (\mr_1(f), \cdots, \mr_p(f)).$$ We refer to $\MRel(f)$ as the \emph{model reliance vector} of model $f$.
\end{definition}

\subsection{Variable Importance Cloud and Diagram}
For a single model $f \in \F$, we compute its model reliance vector $\MRel(f)$, which shows how important the features are to the \emph{single} model. But usually, there is no clear reason to choose one model over another equally-accurate model. Thus, model reliance hides how important a variable \textit{could be}. Accordingly, it hides the \textit{joint} importance of multiple variables. Variable Importance Clouds explicitly characterize this joint importance of multiple variables. The \emph{Variable importance cloud} (VIC) consists of the set of model reliance vectors for all predictive models in the Rashomon set $\Rash$.
\begin{definition}[VIC]\label{VIC}
The \emph{Variable Importance Cloud} of the Rashomon set $\Rash = \Rash(\epsilon, f^*, \F)$ is given by $\VIC(\Rash) = \{\MRel(f): f \in \Rash\}.$
\end{definition}

The VIC is a set in the $p$-dimensional space. We project it onto lower dimensional spaces for visualization. We construct a collection of such projections, referred to as the \emph{Variable Importance Diagram} (VID). Both the VIC and VID embody rich information. This argument will be illustrated with concrete applications later.

\subsection{Rashomon Set and VIC for Ridge Regression Models}

Fix a random vector $(X,Y)$. For a linear regression model $f_\beta \in \F_{lm}$, the expected ridge regression loss is given by
\begin{align*}
L(f_\beta) = & \E[(Y-X^T \beta)^2 + c \norm{\beta}^2] \\ 
= & \E[Y^2] - 2\E[YX^T]\beta + \beta^T \E[XX^T + cI]\beta.
\end{align*}
Given a benchmark model $f_{\beta^*}\in \F_{lm}$, a factor $\epsilon > 0$, following Definition \ref{Rashomon}, the {\em Rashomon set for linear models} $\Rash_{lm}$ is defined as 
$$\Rash_{lm}(\epsilon, f_{\beta^*}, \F_{lm}) = \{f \in \F_{lm} | L(f) \leq (1 + \epsilon)L(f_{\beta^*}) \}.$$
That is, a linear model $f_\beta$ is in the Rashomon set if it satisfies
\begin{equation}\label{Rash_lm}
\beta^T \E[XX^T+ cI]\beta - 2\E[YX^T]\beta + \E[Y^2] \leq (1 + \epsilon) L(f_{\beta^*}).
\end{equation}

Observe that if the random vector $(X,Y)$ is normalized so that the expectation is zero, then $\E(XX^T) = \textit{Var}(X)$ captures the covariance structure among the features, and $\E(YX^T) = \textit{Cov}(Y, X)$ captures the covariance between the outcome and the features. Therefore, the Rashomon set for ridge regression models can be expressed as a function of only these covariances.

Model reliance function $\MRel$ in Definition \ref{MR} turns out to have a specific formula for ridge regression models, given by the lemma below, which is a generalization of Theorem 2 of \cite{FRD2018} to ridge regression.

\begin{lemma}\label{MR_lm}
Given a random vector $(X,Y)$ and the least squares loss function $L$, for $j = 1, 2, \cdots, p$,
\begin{equation}\label{mr_lm}
\mr_j^{\text{diff}}(f_\beta) = 2\Cov(Y, X_j)\beta_j - 2 \beta_{\backslash j}^T \Cov(X_{\backslash j}, X_j)\beta_j,
\end{equation}
As a result, the model reliance function for linear models becomes $$\MRel(f_\beta) = (\mr_1(f_\beta), \cdots, \mr_p(f_\beta)).$$
\end{lemma}

Note that the function $\MRel$ is non-linear in $\beta$. With a slight abuse of notation, we define $\MRel^{-1}$ as the inverse function that maps variable importance vectors to coefficients of a linear model (rather than the model itself). That is, $\MRel^{-1}(\mr_1(f_\beta), \cdots, \mr_p(f_\beta)) = \beta$ instead of $f_\beta$. We assume the existence of the inverse function $\MRel^{-1}$.

With the expressions for both the Rashomon set (Equation \ref{Rash_lm}) and the model reliance function (Equation \ref{mr_lm}), we can characterize the VIC for linear models. 

\begin{theorem}[VIC for Linear Models]\label{VIC_lm_thm}
Fix a benchmark model $f_{\beta^*}\in \F_{lm}$, and a factor $\epsilon > 0$. Let $\VIC = \VIC(\Rash_{lm}(\epsilon, f_{\beta^*}, \F_{lm}))$. Then a vector $\mr \in \VIC$ if it satisfies 
\begin{align}
\MRel^{-1}(\mr)^T \E[XX^T+ cI]\MRel^{-1}(\mr) - 2\E[YX^T]\MRel^{-1}(\mr) + \E[Y^2] \leq (1 + \epsilon)L(f^*).\label{VIC_lm}
\end{align}
\end{theorem}

The theorem suggests that the VIC for linear models depends solely on the covariance structure of the random vector $[X,Y]$, which includes $\E[XX^T]$, $\E[Y^2]$, and $\E[YX^T]$. (The function $\MRel^{-1}$ also depends solely on the covariance structure of $[X,Y]$.)

\subsection{Scale of Data}

In this subsection, we set the regularization parameter $c$ to be 0. We are interested in how the VIC is affected by the scale of our data $[X, Y]$. We prove that the VIC is scale-invariant in features $X$. Rescaling the outcome variable $Y$ does affect the VIC, as it should.

\begin{corollary}[Scale of VIC]\label{scale}
Let $\tilde{X_i} = s_i X_i$ with $s_i > 0$ for all $i = 1, \cdots, p$, and $\tilde{Y} = t Y$. It follows that 
$$\mr \in \VIC(X, Y) \text{ if and only if } \, t^2 \cdot \mr \in \VIC(\tilde{X}, \tilde{Y}),$$ where $\VIC(X,Y)$ denotes the VIC with respect to the Rashomon set $\Rash(\epsilon, f_\beta^*, \F_{lm}; X, Y)$ with $\epsilon > 0$ and $f_{\beta^*}$ being the model that minimizes the expected loss with respect to $[X,Y]$, and $\VIC(\tilde{X}, \tilde{Y})$ is defined in the same way for the scaled variable $[\tilde{X}, \tilde{Y}]$.
\end{corollary}

The proof of the corollary is given in Appendix \ref{AppendixA}. This corollary suggests that the importance of a feature does not rely on its scale, in the sense that rescaling a feature does not change the reliance of any good predictive model on the feature. (In contrast, recall that the magnitudes of the coefficients are sensitive to the scale of the data.) 

\subsection{Special Case: Uncorrelated Features}

As Equation \ref{VIC_lm} suggests, to analyze the VIC for linear models, the key is to study the inverse model reliance function $\MRel^{-1}$. Unfortunately, due to the non-linear nature of $\MRel$, it is difficult to get a closed-form expression of the inverse function in general. In this section, we focus on the special case that all the features are uncorrelated in order to understand some properties of the VIC, before proceeding to the correlated case in later subsections.

\begin{corollary}[Uncorrelated features]\label{VIC_lm_special_thm}
Suppose $\E(X_iX_j) = 0$ for all $i \neq j$. Let $L^* = \min_{f\in\F_{lm}}L(f)$ be the minimum expected loss within the class $\F_{lm}$, and choose the minimizer $f^*$ as the benchmark for the Rashomon set $\Rash = \Rash(\epsilon, f^*, \F_{lm})$. Then the VIC for linear models, $\VIC(\Rash)$, is an ellipsoid centered at $$\mr^* = \left(\frac{2\E[X_1Y]^2}{\Var(X_1) + c}, \cdots, \frac{2\E[X_pY]^2}{\Var(X_p) + c} \right),$$
with radius along dimension $j$ as follows:
$$r_j = 2\E[X_jY] \sqrt{\frac{\epsilon L^*}{\Var(X_j) + c}}.$$ Moreover, when the regularization parameter $c$ is 0, $$r_i > r_j \text{ if and only if }\rho_{iY} > \rho_{jY},$$
where $\rho_{jY}$ is the correlation coefficient between $X_j$ and $Y$.
\end{corollary}

The proof of Corollary \ref{VIC_lm_special_thm} is given in Appendix \ref{AppendixB}. The corollary suggests that the VIC for linear models with uncorrelated features is an ellipsoid that parallels the coordinate axes. 

This result is useful. First, it pins down the variable importance vector $\mr^*$ for the best linear model. Second, for any accurate model, it states that the reliance on feature $j$ is bounded by $\frac{2\E[X_jY]^2}{\Var(X_j) + c} \pm 2\E[X_jY] \sqrt{\frac{\epsilon L^*}{\Var(X_j) + c}}$. Third, within the set of models that have the same expected loss, the surface of the ellipsoid tells how a reduction in the reliance on one feature can be compensated by the increase in the reliances on other features.

\subsection{Approximation of VIC with Correlated Features}\label{VIC_approx_section}

We now proceed to the general case of correlated features. The key difference is that the $\MRel$ function defined by Equation \ref{mr_lm} is no longer linear. As a result, the VIC is no longer an ellipsoid. 

While it always works to numerically compute the true VIC, from which we can directly get (1) the model reliance vector for the best linear model and (2) the bounds for the reliance on each feature for any model in the Rashomon set, it is hard to see how the reliances on different features change when we switch between models with the same loss (which is revealed by the surface of the VIC). To that end, we propose a way to approximate the VIC as an ellipsoid. Under the approximation, we can at least numerically compute the parameters of the ellipsoid, including the center, radii, and how it is rotated. We also comment on the accuracy of the approximation.

Observe that Equation \ref{mr_lm} is a quadratic function of $\beta$. By invoking Taylor's theorem, we have 
\begin{equation}\label{mr_lm2}
\mr_j(\beta) - \mr_{j}(\bar{\beta}) = \nabla^T \mr_j(\bar{\beta}) (\beta- \bar{\beta}) + \frac{1}{2}(\beta- \bar{\beta})^T H_j(\bar{\beta})(\beta- \bar{\beta}),
\end{equation} 
where $\bar{\beta} \in \R^p$ is an arbitrary vector, 
$$\nabla \mr_j(\bar{\beta}) = 
\begin{bmatrix} 
-2 \Cov(X_1, X_j) \bar{\beta}_j \\
\cdots\\
-2 \Cov(X_{j-1}, X_j) \bar{\beta}_j \\
2 \Cov(Y, X_j) - 2 \bar{\beta}_{\backslash j}^{T} \Cov(X_{\backslash j},X_j)\\
-2 \Cov(X_{j+1}, X_j) \bar{\beta}_j \\
\cdots\\
-2 \Cov(X_{p}, X_j) \bar{\beta}_j \\
\end{bmatrix},$$
and $H_j(\bar{\beta})$ is the Hessian matrix that depends only on the covarience structure of the features. (The exact expression is omitted here.)


Equation \ref{mr_lm2} is accurate since there are no higher order terms in Equation \ref{mr_lm}. The quadratic term $\frac{1}{2}(\beta- \bar{\beta})^T H_j(\bar{\beta})(\beta- \bar{\beta})$ in Equation \ref{mr_lm2} is small if either $(\beta- \bar{\beta})$ is small or the Hessian matrix $H_j$ is small. The former happens when we focus on small Rashomon sets and the latter happens when the features are less correlated. In both cases, approximating $\mr_j$ with only the linear term in Equation \ref{mr_lm2} would be close to the original function $\mr_j$. 

If we ignore the higher order term, the relationship between the model reliance vector $\MRel(\beta)$ and the coefficients $\beta$ can be more compactly written with the Jacobi matrix $J(\bar{\beta})$,
\begin{equation*}\label{mr_lm_approx}
\mr(\beta) - \mr(\bar{\beta}) = J(\bar{\beta}) (\beta - \bar{\beta}),
\end{equation*}
where the $i$th row of $J$ is $\nabla^T \mr_j(\bar{\beta})$.
That is,
\begin{align*}\label{jacobi}
J(\bar{\beta}) = 2 \cdot 
\begin{bmatrix} 
\sigma_{Y,1} - \sum_{i \neq 1} \sigma_{i,1}\bar{\beta}_i & -\sigma_{2,1} \bar{\beta}_1 & \cdots & -\sigma_{p,1} \bar{\beta}_1 \\
-\sigma_{1,2} \bar{\beta}_2 & \sigma_{Y,2} - \sum_{i \neq 2} \sigma_{i,2}\bar{\beta}_i & \cdots & -\sigma_{p,2} \bar{\beta}_2 \\
\cdots & \cdots & \cdots & \cdots \\
\cdots & \cdots & \cdots & \cdots \\
-\sigma_{1,p} \bar{\beta}_p & -\sigma_{2,p} \bar{\beta}_p & \cdots & \sigma_{Y,p} - \sum_{i \neq p} \sigma_{i,p}\bar{\beta}_i
\end{bmatrix}
,
\end{align*}
where $\sigma_{i,j} = \Cov(X_i, X_j)$ and $\sigma_{Y,i} = \Cov(Y, X_i)$.

We assume the Jacobi matrix is invertible. (Cases where this would not be true are, for instance, cases where $\Cov(X, Y)$ are all 0, which means there no signal for predicting $Y$ from the $X_i$'s.) Then we can linearly approximate the inverse $\MRel$ function as follows. 

\begin{definition}\label{MR_approx}
For an arbitrary vector $\bar{\beta} \in \R^p$, the approximated $\MRel^{-1}$ is given by 
\begin{equation*}
\MRel^{-1}(\mr) \approx \bar{\beta}+ J^{-1}(\bar{\beta})(\mr - \overline{\mr}),
\end{equation*}
where $\overline{\mr} = \MRel(\bar{\beta})$. 
\end{definition}

We can choose any $\bar{\beta}$ to approximate $\MRel^{-1}$, and we should choose it depending on our purpose. Suppose we are interested in the approximation performance at the boundary of the Rashomon set, it makes sense to pick $\bar{\beta}$ that lies on the boundary. Instead, for overall approximation performance, we should choose $\bar{\beta} = \beta^*$, which is the vector that minimizes expected loss. We can apply Definition \ref{MR_approx} to Theorem \ref{VIC_lm_thm} as follows.

\begin{theorem}\label{VIC_lm_approx_thm}
Fix a benchmark model $f_{\beta^*}\in \F_{lm}$, a factor $\epsilon > 0$. Pick a $\bar{\beta} \in \R^p$. A vector $\mr$ is in the approximated VIC if it satisfies
\begin{align}\label{VIC_lm_approx}
\widetilde{\mr}^T J^{-T} \E [XX^T + cI] J^{-1} \widetilde{\mr} + 2(\bar{\beta}^T\E[XX^T+cI] - \E[YX^T])J^{-1}\widetilde{\mr} +  L(f_{\bar{\beta}})\leq L(f_{\beta^*})(1+\epsilon),
\end{align}
where $\widetilde{\mr} = \mr - \overline{\mr}$.
\end{theorem}
The theorem suggests that the approximated VIC is an ellipsoid. Therefore, we can study its center and radii and perform the same tasks as mentioned in the previous subsection. More details are provided in Appendix \ref{AppendixC}, namely the formula for the ellipsoid approximation of the VIC for correlated features. In what follows, we discuss the accuracy of the approximation.

Recall that from Equation \ref{mr_lm2} we have $\widetilde{\mr}_j = \nabla^T \mr_j(\beta^*) \tilde{\beta} + \frac{1}{2}\tilde{\beta}^T H_j \tilde{\beta}$, where $\widetilde{\mr}_j = \mr_j(\beta) - \mr_j(\bar{\beta})$ and $\tilde{\beta} = \beta - \bar{\beta}$. By dropping the second order term, we introduce the following error for $\widetilde{\mr}_j$,
\begin{equation*}
err_j = \frac{1}{2}\tilde{\beta}^T H_j\tilde{\beta} = -\sum_{i \neq j} \sigma_{ij}\tilde{\beta_i}(\tilde{\beta_i} + \tilde{\beta_j}),
\end{equation*}
where $\sigma_{ij} = \E(X_iX_j)$.

Note that $\abs{\tilde{\beta}_j}$ is bounded by the radius of the Rashomon ellipsoid along dimension $j$,
denoted by $l_j$. Then it follows that 
$$\abs{err_j} \leq \sum_{i \neq j} \abs{\sigma_{ij}}l_i(l_i + l_j).$$
It demonstrates our intuition that the approximation is more accurate when there is less correlation among the features or when the Rashomon set is smaller.

\subsection{2D Visualization of VIC for Linear Models}

We visualize the VIC for linear models in the simplest 2-feature case for a better understanding of the VIC. Let $Z = (Y, X_1, X_2) \in \R^3$. We normalize the variables so that $\E(Z) = \mathbf{0}$. It follows that
$$\Var(Z) = \E(ZZ^T) = \E
\begin{bmatrix}
Y^2 & YX_1 & YX_2 \\
YX_1 & X_1^2 & X_1X_2  \\
YX_2 & X_2X_1 & X_2^2
\end{bmatrix}.$$
Recall that Theorem \ref{VIC_lm_thm} and Lemma \ref{MR_lm} suggest that the VIC is completely determined by the matrix $\Var(Z)$. Moreover, by Corollary \ref{scale}, we can assume without loss that $\sigma_{11} = \sigma_{22} = \sigma_{YY} = 1$. Effectively, the only parameters are the correlation coefficients $\rho_{12}$, $\rho_{1Y}$, and $\rho_{2Y}$, which are the covariates $\sigma_{ij}$ normalized by standard deviations $\sqrt{\sigma_{ii}\sigma_{jj}}$.

We visualize the VIC with regularization parameter $c = 0$. As is discussed above, for larger $c$ the Rashomon set has a smaller size, so that the VIC is closer to an ellipse.

\begin{figure}[ht]
\centering
\includegraphics[scale = 0.3]{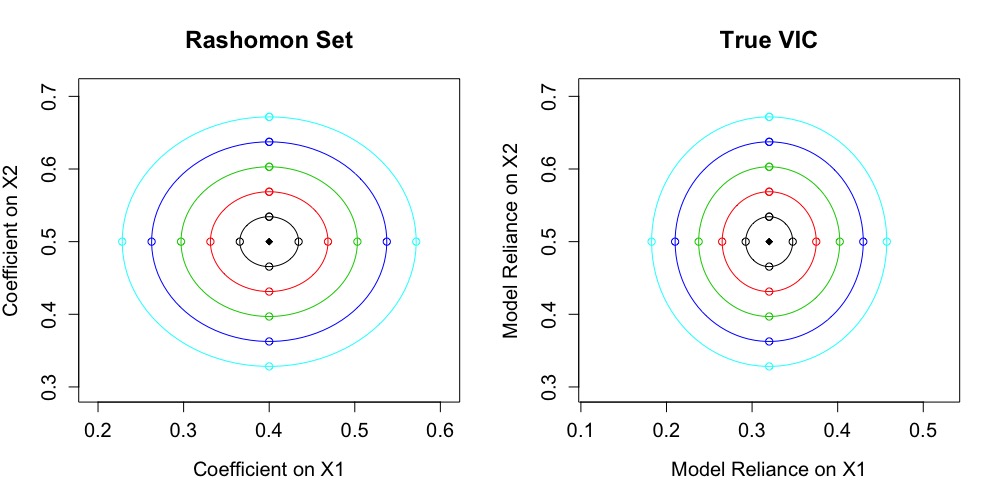}
\caption{The VIC for uncorrelated features: $\rho_{12} = 0$, $\rho_{1Y} = 0.4$, $\rho_{2Y} = 0.5$. The outer curve corresponds to the Rashomon set and the VIC with $\epsilon = 0.05$, and the inner curves correspond to smaller $\epsilon$'s.}
\label{VIC_uncorrelated}
\end{figure}

Figure \ref{VIC_uncorrelated} visualizes the special case where the features are uncorrelated. The left panel of Figure \ref{VIC_uncorrelated} is the Rashomon set. The axes are the values of the coefficients. The Rashomon set is centered at the coefficient of the best linear model. Each ellipse is an iso-loss curve and the outer curves have larger losses. The right panel of Figure \ref{VIC_uncorrelated} is the VIC, with the axes being the model reliances on the features. The center point is model reliance vector of the best linear model, and each curve corresponds to a Rashomon set in the left panel. As is pointed out in Corollary \ref{VIC_lm_special_thm}, when the features are uncorrelated, the VIC is an ellipsoid. We also observe that the VIC ellipses are narrower along $X_1$ than $X_2$, since $\rho_{1Y}<\rho_{2Y}$, which also demonstrates the result in Corollary \ref{VIC_lm_thm}.

When the features are correlated, the VIC is no longer an ellipsoid. We can see from Figure \ref{VIC_small_correlation} below that indeed this is the case. The upper left panel contains the Rashomon sets and the upper right panel contains the corresponding VIC's. Since there is not much correlation between $X_1$ and $X_2$, the VIC's are close to ellipses, especially if we are interested in the inner ones which correspond to smaller Rashomon sets.

As before, we may be interested in approximate the VIC with an ellipse. The lower left panel of Figure \ref{VIC_small_correlation} is the approximated VIC where we invoke Taylor's theorem at the center of the Rashomon set. We can see that the approximated VIC, which is represented by the dashed curve, is indeed close to the true VIC. If we are interested in the performance at the boundary, we may want to invoke Taylor's theorem at the boundary of the Rashomon set, which is visualized by the lower right panel, for four different points on the boundary.

\begin{figure}[ht]
\centering
\includegraphics[scale = 0.3]{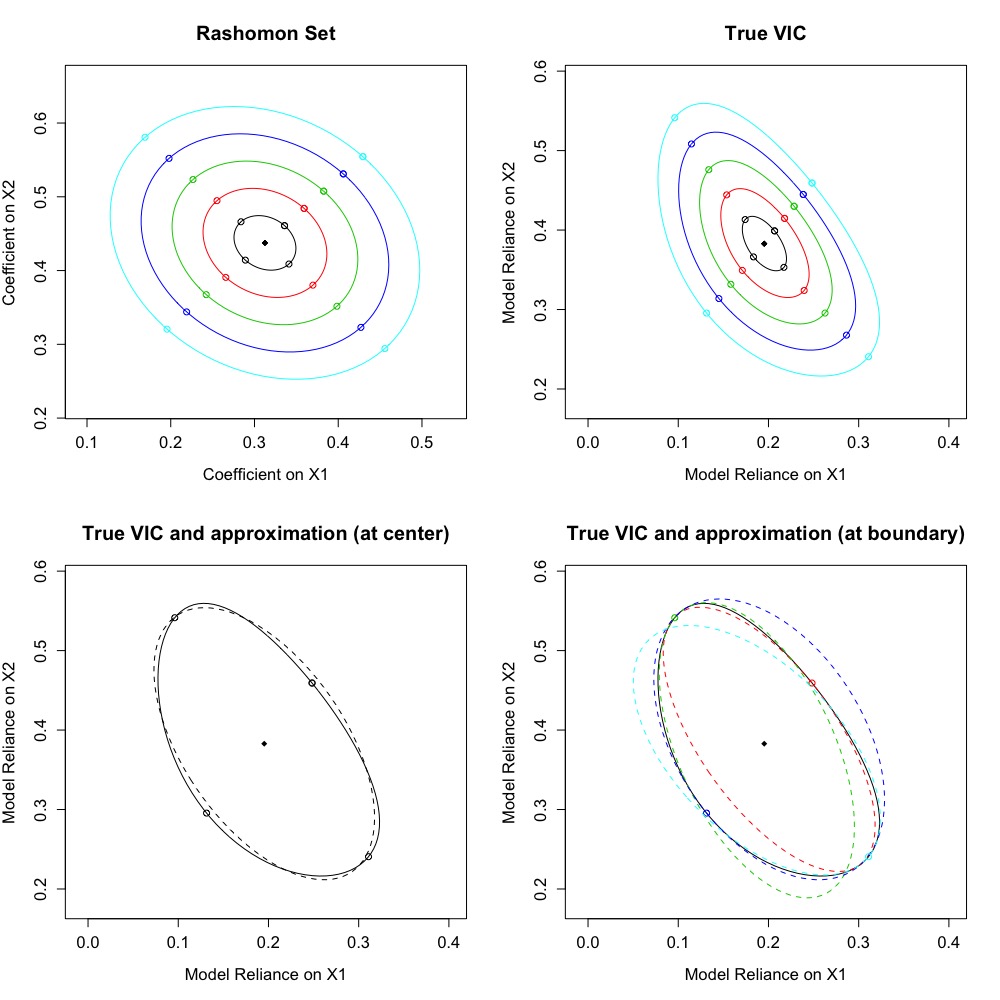}
\caption{The VIC for correlated features: $\rho_{12} = 0.2$, $\rho_{1Y} = 0.4$, $\rho_{2Y} = 0.5$. The outer curve corresponds to the Rashomon set and VIC with $\epsilon = 0.05$, and the inner curves correspond to smaller $\epsilon$'s.}
\label{VIC_small_correlation}
\end{figure}

The VIC can no longer be treated as an ellipse when there is large correlation in the features, and the approximation is far from accurate. This is illustrated by Figure \ref{VIC_large_correlation} below.

\begin{figure}[ht]
\centering
\includegraphics[scale = 0.3]{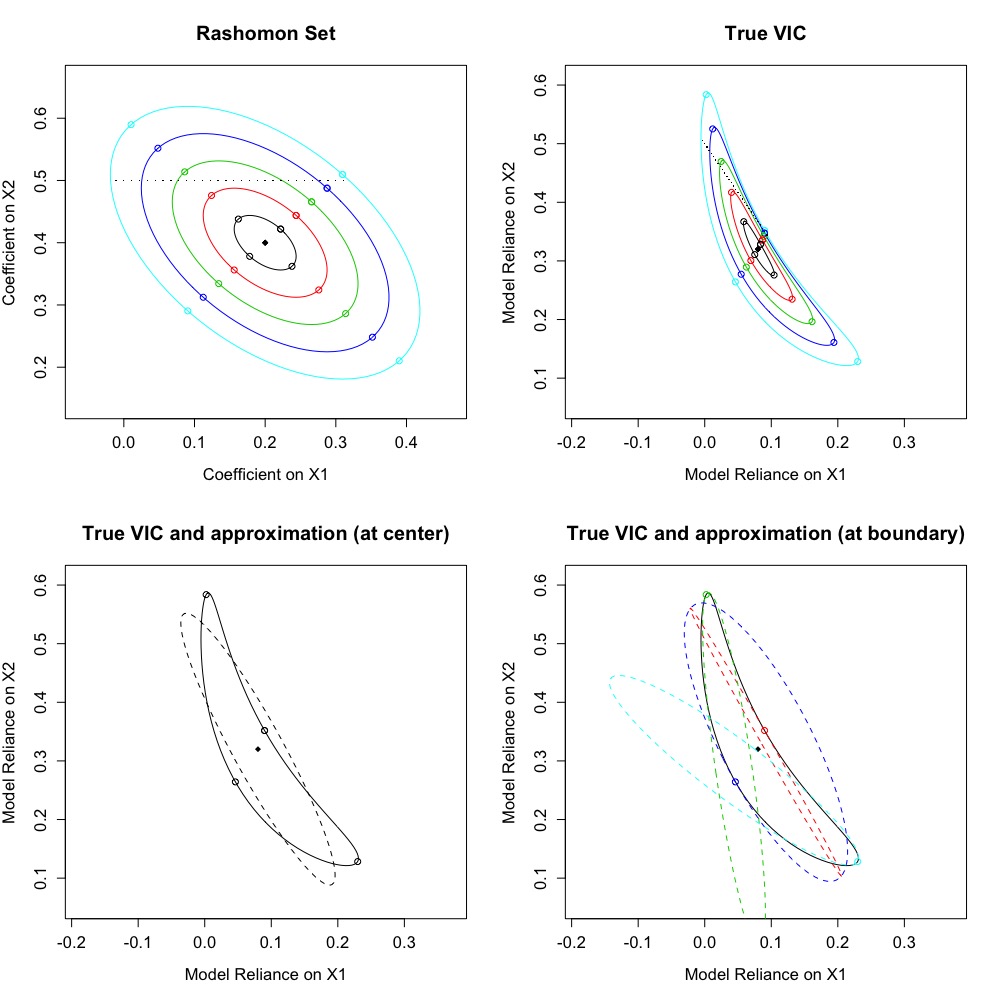}
\caption{The VIC for correlated features: $\rho_{12} = 0.5$, $\rho_{1Y} = 0.4$, $\rho_{2Y} = 0.5$. The outer curve corresponds to the Rashomon set and VIC with $\epsilon = 0.05$, and the inner curves correspond to smaller $\epsilon$'s. }
\label{VIC_large_correlation}
\end{figure}

\section{VIC for Non-linear Problems}

Now that we understand the VIC for linear models, we will apply our analysis to broader applications. 

Our analysis for linear models has made clear that to study the VIC, there are two key ingredients: (1) finding the Rashomon set and (2) finding the $\MRel$ function. We discuss the algorithm for finding the $\MRel$ function in details here in the context of general problems, before proceeding to the algorithm for finding the Rashomon set, which can only be done case-by-case depending on the class of predictive models we are interested in.

\subsection{Finding the $\MRel$ function}

We adopt the following procedure to compute empirical model reliance. Recall that the (population) reliance of model $f$ on feature $X_j$ is defined as the ratio of $L(f; [X_{\backslash j}, \bar{X}_j], Y)$ and $L(f; X, Y)$. The latter is the original expected loss, which can be computed with its empirical analog. The former is the shuffled loss after replacing the random variable $X_j$ with its i.i.d$.$ copy $\bar{X}_j$. We permute the $j^{th}$ column of our dataset $\X$ and compute the empirical loss based on the shuffled dataset. By averaging this empirical shuffled loss several times with random permutations, the average shuffled loss should well-approximate the expected shuffled loss.

While this method works for general datasets, in some applications with binary datasets (both features and outcome are binary variables), the empirical model reliance can be computed with a simpler method. Suppose we are interested in the reliance of a predictive model on variable $X_j$. Compute the loss $L_0$ when $X_j = 0$ and the loss $L_1$ when $X_j = 1$. Find the frequency $p_j$ for $X_j$ to be $1$ in the dataset. Then the shuffled loss is $p_j L_1 + (1-p_j)L_0$.

\subsection{Finding the Rashomon set}

We discuss how to find the Rashomon set for non-linear problems, including logistic regression models and decision trees. For logistic regression models, we approximate the Rashomon set by an ellipsoid, through a sampling step followed by principal components analysis to create the ellipsoid. The sampling and PCA steps are repeated several times until the estimate becomes stable. For decision trees, we consider the case where the data are binary. We find the Rashomon set in the following way: We start with the best decision tree and start to flip the predictions in its leaf nodes. We start with the leaf node with the minimal incremental loss and stop when the flipped decision tree is no longer considered to be good (according to the Rashomon set definition).

\subsubsection{Logistic Regression}

For a dataset $(\X, \y)$ with $\X \in \R^{n \times p}$ and $\y \in \{-1, 1\}^n$, we define the following empirical logistic loss function
$$L(\beta; \X, \y) = \sum_{i=1}^n \log \left(1 + \exp(-y_i \beta^Tx_i)\right),$$
where $\beta \in \R^p$ and $x_i$ is the $i^{th}$ row of $\X$. 

We consider the logistic model class $$\F_{logistic} = \left\{f_\beta: \R^p \rightarrow \R | f_\beta(x) = \frac{1}{1 + e^{\beta^T x}}\right\}.$$ Notice that we can identify this set with $\R^p$, since every logistic model $f \in \F_{logistic}$ is completely characterized by $\beta \in \R^p$. Therefore, we define $\F = \R^p$ instead to represent parameter space.
Let $\beta^*$ be the coefficient that minimizes the logistic loss. Then the Rashomon set $\Rash = \Rash(\epsilon, f_{\beta^*}, \F)$ and the variable importance cloud $\VIC = \VIC(\Rash)$ is given by Definitions \ref{Rashomon} and \ref{VIC}. Let us go through these steps for logistic regression.

The empirical $\MRel$ function is introduced above. Let us now find the Rashomon set. Note that there is no closed-form expression for the Rashomon set for logistic regression models, but it is convex. We approximate it with a $p$-dimensional ellipsoid in $\R^p$. Under this approximation, we can sample the coefficients from the Rashomon set, and proceed to the next step of VIC analysis. Below is the algorithm we use to approximate the Rashomon set with an ellipsoid.
\begin{enumerate}
\item Find the best logistic model $\beta^*$ that minimizes the logistic loss. Let $L^*$ be the minimum loss.
\item Initial sampling: Randomly draw a set of $N$ coefficients in a ``box'' centered at $\beta^*$. Eliminate the coefficients that give logistic losses that exceed $(1+\epsilon)L^*$.
\item PCA: Find the principle components. Compute the center, radii of axes, and the eigenvectors. To get the boundary of the Rashomon set, resample $N$ coefficients from a slightly larger ellipsoid with radii multiplied by some scaling factor $r>1$. The sampling distribution is a $\beta(1,1)$ distribution along the radial axis in order to get more samples closer to the boundary. Eliminate the coefficients that give logistic losses that exceed $(1+\epsilon)L^*$.
\item Repeat the third step $M$ times.
\end{enumerate}

There are four parameters in the whole process, the number of coefficients $N$ to sample in each step, the size of the box for initial sampling, the scaling factor $r$ that scales the ellipsoid, and the number of iterations $M$. The parameter $N$ need to be a large number for robustness, but not too large for computation. The size of the box should be large enough to include some points outside the Rashomon set, so as to get a rough boundary for the Rashomon set. The same applies to the factor $r$. We need to repeat $M$ times to get a stable boundary of the Rashomon set. 

In our experiments (Section 5), we set $N = 500$ and set the size of the box as a certain factor times the standard deviation of the logistic estimator $\beta^*$ so that about $75\%$ of the sampled coefficients survive the elimination in the initial round. 

We tune the other two parameters to get a robust approximation of the Rashomon set. Suppose we want to choose the optimal scaling factor and number of iterations $(r^*, M^*)$ from a set of candidates. Let $\bar{r}$ be an upper bound of the scaling factors. For each candidate $(r, M)$, we implement the above algorithm and get the resulting ellipsoid. We sample $N$ coefficients from this ellipsoid scaled by $\bar{r}$, and count the number of coefficients that remain in the Rashomon set and compute the survival rate. This number is related to the performance of the algorithm with parameter $(r, M)$.

We first discuss the consequence of changing the scaling factor $r$. If the factor $r$ is too small, every coefficient is in the Rashomon set. Effectively, we are sampling from a strict subset of the Rashomon set, even though we scaled it by the factor $\bar{r}$, so that the survival rate is 1. If we use the algorithm with this pair of parameters $(r, M)$, we only get a subset of the Rashomon set. Hence the approximation is not accurate. As $r$ increases, the ellipsoid that approximates the Rashomon set grows bigger. As we test the performance by sampling from the ellipsoid scaled by the upper bound on the scaling factor, namely $\bar{r}$, we are sampling from a superset of the Rashomon set. Only a fixed portion of the sampled points are in the Rashomon set, because both the Rashomon set and $\bar{r}$ are fixed.

As $r$ grows, the survival rate would first decrease and then become stable. The factor $r$ at which the survival rate becomes stable should be used in the algorithm. This is because with this factor, we get the boundary of the Rashomon set. The resulting ellipsoid well-approximates the Rashomon set. Figure \ref{robustness} below demonstrates the argument.

Now we discuss the consequence of changing $M$. Due to the initial sampling in a box, we expect that it takes several iterations to approximate the Rashomon set. Therefore, the survival rate may change when $M$ is small. On the other hand, when $M$ becomes large, the survival rate should not change, since effectively we are repeatedly sampling from the same ellipsoid. We should pick the value of $M$ at which the survival rate becomes stable.

\begin{figure}
\centering
\includegraphics[scale = 0.55]{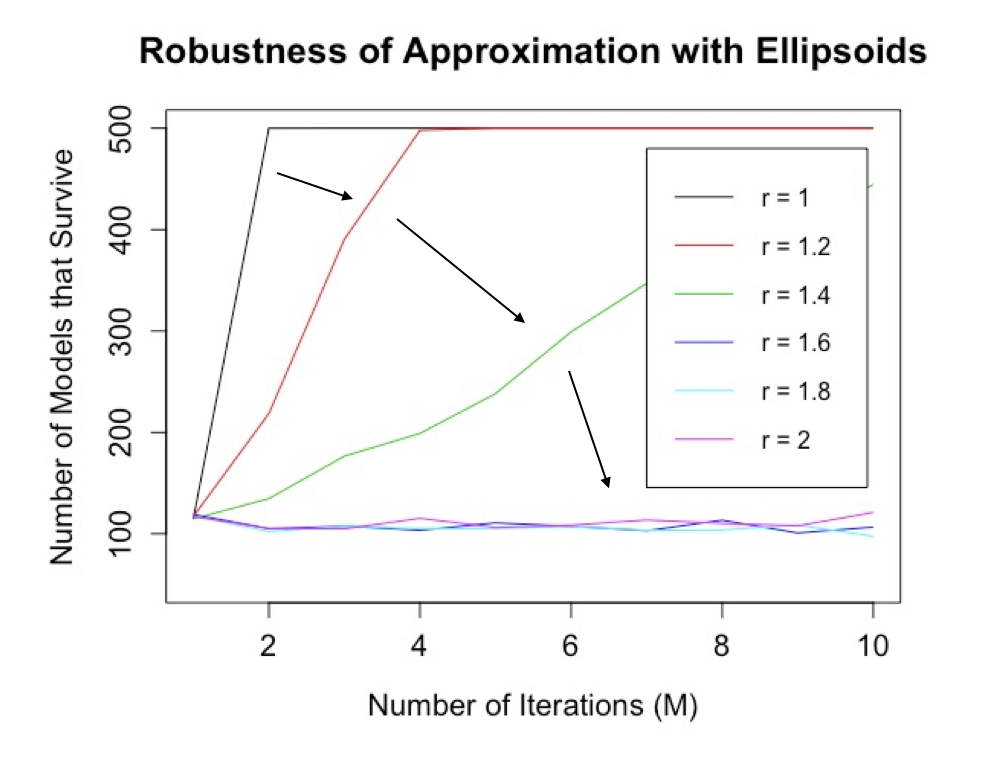}
\caption{Tuning the parameters $(r, M)$. As $r$ increases, the number of points that survive the test decreases and becomes stable. Similarly, when $M$ increases, the number of points that survives the test also becomes stable. The figure is generated by the tuning process for the experiment in Section 5.1.}
\label{robustness}
\end{figure}


\subsubsection{Decision Tree}

In this subsection, we implement the VIC analysis for binary data. The method extends to categorical data as well.

For a binary dataset $(\X, \y)$ with $\X \in \{0,1\}^{n \times p}$ and $\y \in \{-1, 1\}^n$, a decision tree is represented by a function $f: \{0,1\}^p \rightarrow \{-1,1\}$.\footnote{This is actually an equivalent class of decision trees.} A decision tree $f$ splits according to feature $j$ if there exists $x, x' \in \{0,1\}^p$ with $x_j \neq x'_j$ and $x_{-j} = x'_{-j}$, and $f(x) \neq f(x')$. We restrict our attention to the set of trees that splits according to no more than $N$ features, and denote this class by $\F_N$. The purpose is to exclude overfitted trees.

We define loss as misclassification error (0-1 loss). In particular, 
$$L(f; \X, \y) = \sum_{i=1}^n \mathbf{1}[f(x_i) \neq y_i],$$
where $x_i$ is the $i^{th}$ row of $\X$. Let $f^* \in \F$ be the tree in our model class that minimizes the loss. Then we have the Rashomon set $\Rash = \Rash(\epsilon, f^*, \F)$ and the variable importance cloud $\VIC = \VIC(\Rash)$ defined as usual. We use the method described above find the empirical $\MRel$ function and below we describe how to find the Rashomon set.

Again, there is no closed-form expression for the Rashomon set for decision trees. This time, we are going to search for the true Rashomon set, without approximation, using the fact that features are binary. (The same method extends to categorical features.) Suppose we want to find all ``good trees'' that split according to features in the set $\{1, 2, \cdots, m\}$, with $m < p.$ There can be at most $2^m$ unique realizations of the features that affect the prediction of the decision tree. Moreover, there are at most $2^{2^m}$ equivalent classes of decision trees, since the outcome is also binary. The na\"ive method is to compute the loss for each equivalence class of trees, and the collection of ``good trees'' forms the Rashomon set.

While this method illustrates the idea, it is practically impossible for $m$ as low as 4. Alternatively, for each of the $2^m$ unique observations, we count the frequency of $y = 1$ and $y = -1$ and record the \emph{gap} of these counts for each observation. (We will define the gap formally in the next section.) The best tree predicts according to the majority rule. The second and third best trees flip the prediction for the observation with the smallest and second-to-the-smallest gaps. The fourth best tree either flips the prediction for the one with the third-to-the-smallest gap, or for both with the smallest- and second-to-the-smallest gap, whichever is smaller. Searching for trees with this method, we can stop the process once we reach a tree that has more than $(1+\epsilon) L^*$ loss, where $L^*$ is loss of the best tree. This method is computationally feasible.

\subsection{Comparing VICs for Logistic Regression and Decision Trees}

The VIC for decision trees is different from that for logistic models in two ways. First, the VIC for decision trees is discrete. Second, there might be a clustering structure in the VID for decision trees.

To explain the first difference, note that we define model reliance differently. For decision trees, it is defined as the ratio of 0-1 losses before and after shuffling the observations. For logistic models, it is defined as the ratio of logistic losses. While logistic loss is continuous in coefficients for a given dataset, 0-1 loss may jump discretely even for a small modification of the tree. That explains why the VIC for decision trees is discrete. The remainder of this subsection attempts to gain intuition about the clustering structure.

For any possible realization of the features $x \in \{0,1\}^p$, let $\#(x;1)$ be the number of observations with features $x$ and outcome 1. $\#(x;-1)$ is defined similarly. For simplicity, we leave out the sparsity restriction for simplicity. In this case, the best decision tree $f^*$ can be defined as 
\begin{equation*}
f^*(x) = \begin{cases}
1 \quad\;\, \text{ if } \#(x;1) \geq \#(x;-1),\\
-1 \quad \text{otherwise}.
\end{cases}
\end{equation*}
For illustration, we consider the clustering structure along the $\mr_1$ dimension, which pertains to feature $X_1$. Let $L^*$ be the loss associated with $f^*$ and $\mr_1^*$ be its reliance on $X_1$. We now characterize the conditions so that there are clusters of points in the VIC. Fix a vector $\bar{x}_{\backslash 1} \in \{0,1\}^{p-1}$. Let $A(\bar{x}_1)^+ = \#([\bar{x}_1, \bar{x}_{\backslash 1}]; 1)$ and $A(\bar{x}_1)^- = \#([\bar{x}_1, \bar{x}_{\backslash 1}]; -1)$ for $\bar{x}_1 \in \{0,1\}$. For example, $A(1)^+$ is the number of observations with feature $[1, \bar{x}_{\backslash 1}]$ and outcome $1$. Consider the tree $f_{[\bar{x}_1,\bar{x}_{-1}]}$ that satisfies
\begin{equation*}
f_{[\bar{x}_1,\bar{x}_{-1}]}(x) = \begin{cases}
-f^*(x) \quad \text{ if } x = [\bar{x}_1,\bar{x}_{-1}],\\
f^*(x) \quad \quad \text{otherwise}.
\end{cases}
\end{equation*}
That is, $f_{[\bar{x}_1,\bar{x}_{-1}]}$ flips the prediction for the observation $x = [\bar{x}_1,\bar{x}_{-1}]$ only. This tree has a total loss that is larger than $L^*$ by the gap $e$:
$$e = \abs{A(\bar{x}_1)^+ - A(\bar{x}_1)^-}.$$ We assume that $e \leq \epsilon L^*$ so that the tree $f_{[\bar{x}_1,\bar{x}_{-1}]}$ is in the Rashomon set.

Now consider the shuffled loss. Observe that $f^*$ and $f_{[\bar{x}_1,\bar{x}_{-1}]}$ only differs when $x = [\bar{x}_1,\bar{x}_{-1}]$. When computing the shuffled loss, the difference comes from the observations with $x = [\bar{x}_1,\bar{x}_{-1}]$ whose values for feature $X_1$ remain the same after shuffling, and the observations with $x = [1- \bar{x}_1,\bar{x}_{-1}]$ whose values for feature 1 change after shuffling. There are $A(\bar{x}_1)^+$ observations with features $[\bar{x}_1,\bar{x}_{-1}]$ and outcome $1$, and $A(\bar{x}_1)^-$ observations with features $[\bar{x}_1,\bar{x}_{-1}]$ and outcome $-1$. Therefore, the former situation can contribute to the difference in shuffled loss by no more than $e = \abs{A(\bar{x}_1)^+ - A(\bar{x}_1)^-}$ (and the loss for $f_{[\bar{x}_1,\bar{x}_{-1}]}$ is larger). Similarly, the latter situation can contribute to the difference in shuffled loss by no more than $e' =  \abs{A(1-\bar{x}_1)^+ - A(1-\bar{x}_1)^-}$ (yet it is ambiguous whether the loss for $f_{[\bar{x}_1,\bar{x}_{-1}]}$ is larger or smaller than $f$ by $e'$).

Since the best tree $f^*$ has loss $L^*$ and reliance $\mr^*_1$ on $X_1$, its shuffled loss is $\mr^*_1 L^*$. The original loss of $f_{[\bar{x}_1,\bar{x}_{-1}]}$ is $L^* + e$ and its shuffled loss is $\mr^*_1 L^* + pe \pm (1-p)e'$, where $p = Prob (x_1 = \bar{x}_1)$. Therefore, we know that the reliance of the two trees on feature $X_1$ differ by
\begin{align*}
\mr_1 - \mr_1^* = & \frac{\mr^*_1 L^* + pe \pm (1-p)e'}{L^* + e} - \mr_1^* \\
= & \frac{ (p - \mr_1^*)e \pm (1-p)e'}{L^* + e}.
\end{align*}
For the set of decision trees that flip the prediction of $f^*$ at one leaf  and whose increments in loss do not exceed $\epsilon L^*$, if none of them has a large $\mr_1 - \mr_1^*$, then there is no cluster in $\mr_1$ dimension. Otherwise there could be clustering, we show this empirically for a real dataset in Section 5.1.

\section{Ways to Use VIC}

We discuss ways to use the VIC in this section and focus on understanding variable importance in the context of the importance of other variables and providing a context for model selection.

\subsection{Understanding Variable Importance with VIC/VID}

The goal of this paper is to study variable importance in the context of the importance of other variables. We illustrate in this section how VIC/VID achieves this goal with a simple thought experiment regarding criminal recidivism prediction. To provide background, in 2015 questions arose from a faulty study done by the Propublica news organization, about whether a model (COMPAS - Correction Offender Management Profiling for Alternative Sanctions) used throughout the US Court system was racially biased. In their study, Propublica found a linear model for COMPAS scores that depended on race; they then concluded that COMPAS must depend on race, or its proxies that were not accounted for by age and criminal history. This conclusion is based on methodology that is not sound: what if there existed another model that did not depend on race (given age and criminal history), but also modeled COMPAS well?

While we will study the same dataset Propublica used to analyze variable importance for criminal recidivism prediction in the experiment section, we perform a thought experiment here to see how VIC/VID addresses this problem. Consider the following data-generating process. Assume that a person who has committed a crime before (regardless of whether he or she was caught or convicted) is more likely to recidivate, which is independent of race or age. However, for some reason (e.g., discrimination) a person is more likely to be found guilty (and consequently has prior criminal history) if he or she is either young or black. Under these assumptions, there might be three categories of models that predict recidivism well: each relies on race, age or prior criminal history as the most important variable. Thus, it is not sound to conclude without further justification that recidivism depends on race.

In fact, we may find all three categories of models in the Rashomon set. The corresponding VIC may look like a 3d ellipsoid in the space spanned by the importance of race, age and prior criminal history. Note that the surface of the ellipsoid represents models with the same loss. We may find that, staying on the surface, if the importance of race is lower, either the importance of age or prior criminal history is higher. We may conclude that race is important for recidivism prediction only when age and prior criminal history are not important, which is a more comprehensive understanding of the dataset as well as the whole class of well-performing predictive models, compared with Propublica's claim.

For more complicated datasets and models, the VIC's are in higher dimensional spaces, making it hard to make any statement directly from looking at the VIC's. In these situations, we need to resort to the VID. In the context of the current example, we project VIC into the spaces spanned by pairs of the features, namely (age, race), (age, prior criminal history) and (race, prior criminal history). Each projection might look like an ellipse. Under our assumptions regarding the data-generating process, we might expect, for example, a downward sloping ellipse in the (race, prior criminal history) space, indicating the substitution of the importance of race and prior criminal history.

The axes of this thought experiment are the same as those observed in the experiments in Section 5.1; there however, we make no assumption about the data-generating process.

\subsection{Trading off Error for Reliance: Context for Model Selection}

VIC provides a context for model selection. As we argued before, we think of the Rashomon set as a set of almost-equally accurate predictive models. Finding the single best model in terms of accuracy may not make a lot of sense. Instead, we might have other concerns (beyond Bayesian priors or regularization) that should be taken into account when we select a model from the Rashomon set. Effectively, we trade off our pursuit for accuracy for those concerns.

For example, in some applications, some of the variables may not be {\em admissible}. When making recidivism predictions, for instance, we want to find a predictive model that does not rely explicitly on racial or gender information. If there are models in the Rashomon set that have no reliance on both race or gender, we should use them at the cost of reducing predictive accuracy. This cost is arguably negligible, since the model we switch to is still in the Rashomon set. It could be the case that every model in the Rashomon set relies on race to some non-negligible extent, suggesting that we cannot make good predictions without resorting to explicit racial information. While this limitation would be imposed by the dataset itself, and while the trade-off between accuracy and reliance on race is based on modeler discretion, VIC/VID would discover that limitation.

In addition to inadmissible variables, there could also be situations in which we know {\em a priori} that some of the variables are less credible than the others. For instance, self-reported income variables from surveys might be less reliable than education variables from census data. We may want to find a good model that relies less on variables that are not credible. VIC is a tool to achieve this goal. This application is demonstrated in Section 5.2.

\subsection{Variable Importance and Its Connection to Hypothesis Testing for Linear Models}

Recall that model reliance is computed by comparing the loss of a model before and after we randomly shuffle the observations for the variable. Intuitively, this should tell the degree to which the predictive power of the model relies on the specific variable. Another proxy for variable importance for linear models could be the magnitude of the coefficients (assuming features have been normalized). When the coefficient is large, the outcome is more sensitive to changes in that variable, suggesting that the variable is more important. This measure is also connected to hypothesis testing; the goal of this subsection is to illustrate this. 

We first argue that the magnitude of the coefficients is a different measure of variable importance than model reliance. Coefficients do not capture the correlations among features, whereas model reliance does. We illustrate this argument with Figure \ref{VIC_large_correlation}. The dotted line in the upper left panel is the set of models within the Rashomon set that have the same $\beta_2 = 0.5$ and different $\beta_1$. The coefficients might suggest that feature $X_2$ is equally important to each of these models, because $X_2$'s coefficient is the same for all of them. (The coefficient is 0.5.) We compute the model reliance for these models and plot them with the dotted line in the upper right panel of Figure \ref{VIC_large_correlation}. (One can check that these indeed form a line.) This suggests that these models rely on feature $X_2$ to different degrees. This is because the variable importance metric based on coefficients ignores the correlations among features. On the other hand, model reliance on $X_2$ is computed by breaking the connection between $X_2$ and the rest of the data $(Y, X_1)$. One can check that $\mr_2 = 2Cov(Y - X_1\beta_1, X_2\beta_2)$, which intuitively represents the correlation between $X_2$ and the variation of $Y$ not explained by $X_1$. Therefore, this measure is affected by the correlation between $X_1$ and $X_2$.

While one can check whether a variable is important or not by hypothesis testing, this technique relies heavily on parametric assumptions. On the other hand, model reliance does not make any additional assumption beyond that the observations are i.i.d. However, given the {\em same} set of assumptions for testing the coefficients, we can also test whether the model reliance of the best linear model on each feature is zero or not when the regularization parameter $c$ is 0. (See Appendix \ref{AppendixD} for the set of assumptions.) 

\begin{theorem}\label{hypothesis}
Fix a dataset $(\X, \y)$. Let $\hat{\beta} = (\X^T\X)^{-1}\X^T\y$ be the best linear model. Let $\widehat{\MRel}_j: \R^p \rightarrow \R$, the empirical model reliance function for variable $j$, be given by $$\widehat{\MRel}_j(\beta) = 2\widehat{Cov}(Y, X_j)\beta_j - 2 \beta^T \widehat{Cov}(X, X_j)\beta_j + 2 \widehat{Var}(X_j)\beta_j^2,$$
where $\widehat{Cov}$ and $\widehat{Var}$ are the empirical covariance and variance. Let $$\hat{\Sigma} = \nabla^T \widehat{\MRel}_j(\hat{\beta}) \widehat{Var}(\hat{\beta})\nabla \widehat{\MRel}_j(\hat{\beta}),$$
where $\nabla^T \widehat{\MRel}_j$ is the gradient of $\MRel_j$ with the population covariance and variance replaced by their empirical analogs, and $\widehat{Var}(\hat{\beta})$ is the variance of the estimator, which is standard for hypothesis testing. Then,
$$n(\widehat{\MRel}_j(\hat{\beta}) - \MRel_j(\alpha))^T \hat{\Sigma}^{-1}(\widehat{\MRel}_j(\hat{\beta}) - \MRel_j(\alpha)) \overset{d}{\longrightarrow} \chi_1^2,$$
where $\alpha$ is the true coefficient.
\end{theorem}

The proof of Theorem \ref{hypothesis} is given in Appendix \ref{AppendixD}. Let us show how to apply this theorem. Suppose we want to perform the following hypothesis test,
$$H_0: \quad \MRel_j(\alpha) = 0; \quad \quad H_1: \quad \MRel_j(\alpha) \neq 0.$$ That is, suppose we want to test whether variable $j$ is not important at all.
Theorem \ref{hypothesis} implies that under $H_0$, $$\hat{Z}_j := n\widehat{\MRel}_j(\hat{\beta})^T (\nabla^T \widehat{\MRel}_j(\hat{\beta})) \widehat{Var}(\hat{\beta})\nabla \widehat{\MRel}_j(\hat{\beta}))^{-1}\widehat{\MRel}_j(\hat{\beta})\overset{d}{\longrightarrow} \chi_1^2.$$
This allows us to test if the population model reliance for the best linear model on variable $j$ is zero. If variable $j$ is not important, our testing statistic $\hat{Z}_j$ is $\chi_1^2$ distributed.

\section{Experiments}

In this section, we want to apply VIC/VID analysis to real datasets and demonstrate its usage. We work with criminal recidivism prediction data, in-vehicle coupon recommendation data and image classification data in this section.

\subsection{Experiment 1: Recidivism Prediction}

As we introduced before, the Propublica news organization found a linear model for COMPAS score that depends on race, and concluded that it is racially biased. This conclusion is unwarranted, since there could be other models that explain COMPAS well without relying on race. (See also \cite{FBL2016}.)

To investigate this possibility, we study the same dataset of 7214 defendants in Broward County, Florida. The dataset contains demographic information as well as the prior criminal history and 2-year recidivism information for each defendant. Our outcome variable is recidivism, and covariate variables are age, race, gender, prior criminal history, juvenile criminal history, and current charge.\footnote{recidivism = 1 if a defendant recidivates in two years. age = 1 if a defendant is younger than 20 years old. race = 1 if a defendant is black. gender = 1 if a defendant is a male. prior = 1 if a defendant has at least one prior crime. juvenile = 1 if a defendant has at least one juvenile crime. charge = 1 if a defendant is charged with crime.} We explore two model classes: logistic models and decision trees. In our analysis below, we find that in both classes there are indeed models that do not rely on race. Moreover, race tends to be an important variable only when prior criminal history is not important.

\subsubsection{VID for Logistic Regression}

\begin{figure}
\centering
\includegraphics[scale = 0.35]{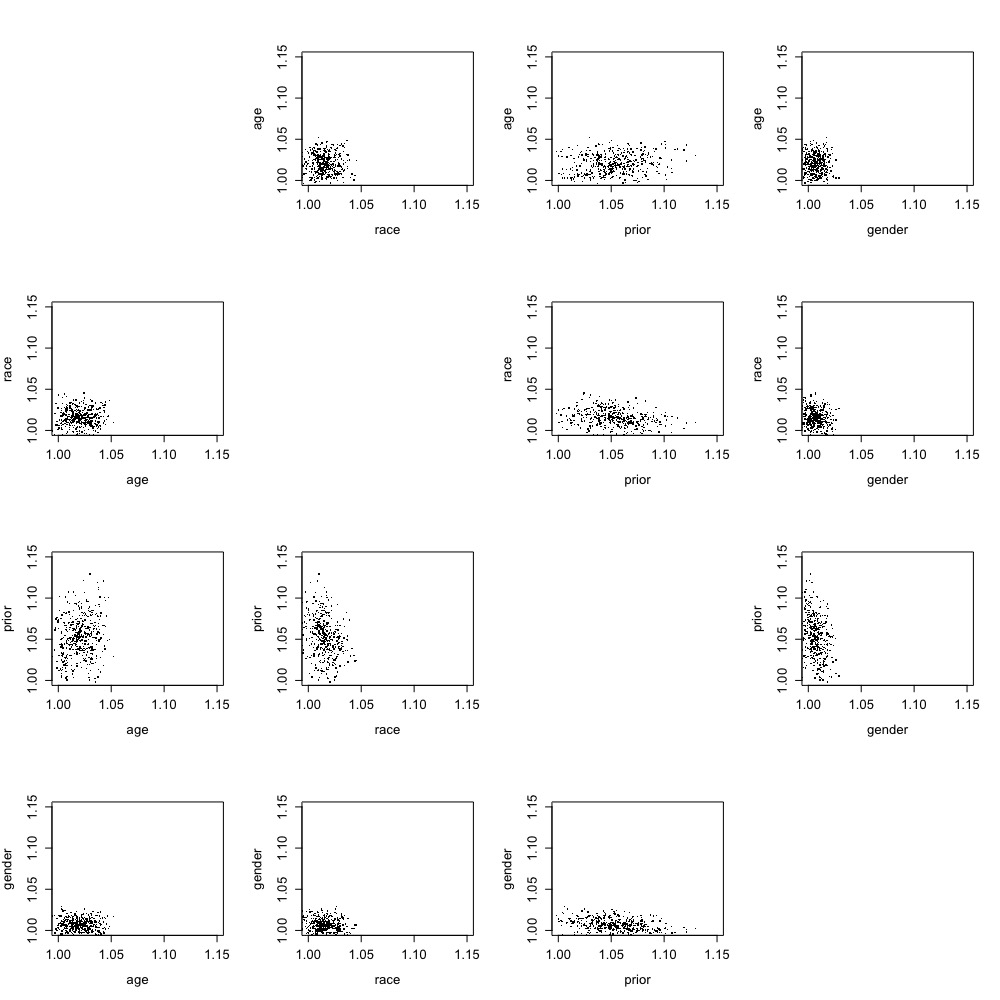}
\caption{VID for Recidivism: logistic regression. This is the projection of the VIC onto the space spanned by the four variables of interest: age, race, prior criminal history and gender. The point, say (1.02, 1.03), in the first diagram in the first row suggests that there is a model in the Rashomon set with reliances 1.02 on race and 1.03 on age.}
\label{VID for Recidivism}
\end{figure}

Since we have 6 variables, the VIC is a subset of $\R^6$. We display only VID (see Figure \ref{VID for Recidivism}) based on four variables: age, race, prior criminal history and gender. 

The first row of the VID is the projection of the VIC onto the space spanned by age and each of the other variables of interest, with the variable importance of age on the vertical axis. We can see that the variable importance of age is roughly bounded by $[1, 1.05]$, which suggests there is no good model that relies on age to a degree more than 1.05, and there exists a good model that does not rely on age. Note that the bounds are the same for any of the three diagrams in the first row.

By comparing multiple rows, we observe that the variable importance of prior criminal history has the greatest upper bound, and the variable importance of gender has the lowest upper bound. Moreover, prior criminal history has the greatest average variable importance and gender has the lowest average importance.  We also find that there exist models that do not rely on each of the four variables. However, the diagrams in the third row reveal that there are only a few models with variable importance of prior criminal history being 1, while the diagrams in the fourth row reveal that there are a lot models with variable importance of gender being 1. All of this evidence indicates that prior criminal history is the most important variable of those we considered, while gender is the least important one among the four.

We now focus on the diagram at Row 3 Column 2, which reveals the variable importance of prior criminal history in the context of the variable importance of race. We see that when importance of race is close to 1.05, which is its upper bound, the variable importance of prior criminal history is in the range of $[1.025, 1.075]$. On the other hand, while the importance of prior criminal history is close to 1.13, which is its upper bound, the variable importance of race is lower. The scatter plot has a slight downward sloping right edge. Since the boundary of the scatter plot represents models with equal loss (because they are on the boundary of the Rashomon set), the downward sloping edge suggests that as we rely less on prior criminal history, we must rely more on race to maintain the same accuracy level. In contrast, the diagram at Row 3 Column 1 has a vertical right edge, suggesting that we can reduce the reliance on prior criminal history without increasing the reliance on age.

\subsubsection{VID for Decision Trees}

\begin{figure}
\centering
\includegraphics[scale = 0.35]{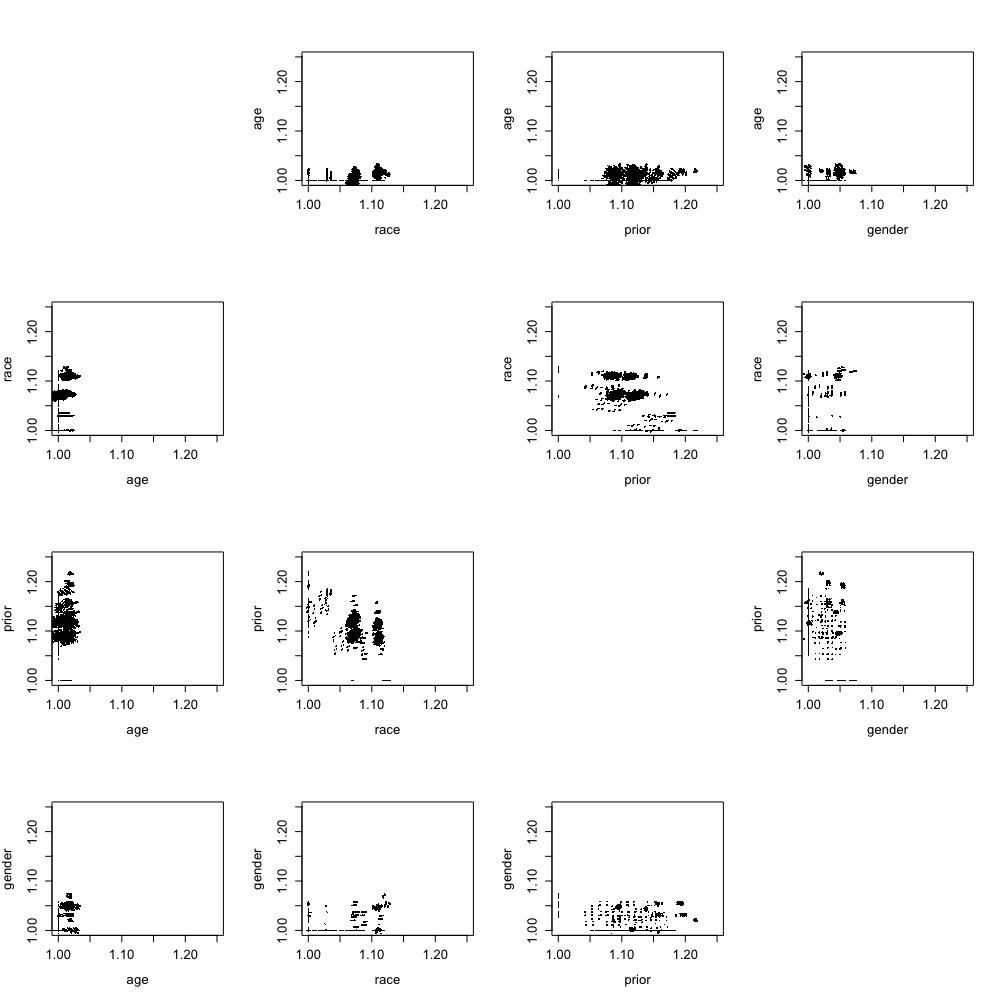}
\caption{VID for Recidivism: decision trees. This is the projective of the VIC onto the space spanned by the four variables of interest: age, race, prior criminal history and gender. Unlike Figure \ref{VID for Recidivism}, the VIC is generated by the Rashomon set that consists of the all the good decision trees instead of logistic regression models. However, the diagrams should be interpreted in the same way as before.}
\label{VID for Recidivism2}
\end{figure}

In this subsection we work on the same dataset but focus on a different class of models, the class of decision trees that split according to no more than 4 features. The restriction put on splitting aims to avoid overfitting. The VID (see Figure \ref{VID for Recidivism2}) for the same four variables of interest is given below. There is a striking difference between the VID for decision trees and logistic models: the former is discrete and has clustering structure. This demonstrates our discussion in Section 3. 

The VID for decision trees also reveals that prior criminal history is the most importance variable for decision trees. However, gender becomes more important than age for decision trees. Figure \ref{VID for Recidivism2} at Row 3 Column 2 regarding the variable importance of prior criminal history and race also suggests a substitution pattern: the importance of prior criminal history is lower when race is important, and vice versa.

\subsection{Experiment 2: In-Vehicle Coupon Recommendation}

In designing practical classification models, we might desire to include other considerations besides accuracy. For instance, if we know that when the model is deployed, one of the variables may not always be available, we might prefer to choose a model that does not depend as heavily on that variable. For instance, let us say we deploy a model that provides social services to children. In the training set we possess all the variables for all of the observations, but in deployment, the school record may not always be available. In that case, it would be helpful, all else being equal, to have a model that did not rely heavily on school record. It happens fairly often in practice that the sources of some of the variables are not trustworthy or reliable. In this case, we may face the same tradeoff between accuracy and desired model characteristics of the variable importance. This section provides an example where we create a trade-off between accuracy and variable importance; among the set of accurate model, we choose one that places less importance on a chosen variable.

We study a dataset about mobile advertisements documented in \cite{WRDLKM2017}, which consists of surveys of 752 individuals. In each survey, an individual is asked whether he or she would accept a coupon for a particular venue in different contexts (time of the day, weather, etc.) There are 12,684 data cases within the surveys.

We use a subset of this dataset, and focus on coupons for coffee shops. Acceptance of the coupon is the binary outcome variable, and the binary covariates include zeroCoffee (takes value 1 if the individual never drinks coffee), noUrgentPlace (takes value 1 if the individual has no urgent place to visit when receiving the coupon), sameDirection (takes value 1 if the destination and the coffee shop are in the same direction), expOneDay (takes value 1 if the coupon expires in one day), withFriends (takes value 1 if the individual is driving with friends when receiving the coupon), male (takes value 1 if the individual is a male), and sunny (takes value 1 if it is sunny when the individual receives the coupon).

We compute the VIC for the class of logistic regression models. Rather than providing the corresponding VID, we display only coarser information about the bounds of the variable importance in Table \ref{coupon_reliance_bound}, sorted by importance. 

\begin{table}[htp]
\begin{center}
\begin{tabular}{c|ccc}
				& upper bound & lower bound  & \\
\hline
zeroCoffee		& 1.31		& 1.19		& more important\\
noUrgentPlace		& 1.16		& 1.06	        & $\uparrow$ \\
sameDirection		& 1.07		& 1.03		& $|$\\
expOneDay		& 1.06		& 1.00		& $|$\\
withFriends		& 1.02		& 1.00		& $|$\\
male				& 1.01		& 1.00		& $\downarrow$\\
sunny			& 1.00		& 1.00		& less important
\end{tabular}
\end{center}
\caption{Bounds on model reliance within the Rashomon set. Each number represents a possibly different model. This table shows the range of variable importance among the Rashomon set.}
\label{coupon_reliance_bound}
\end{table}

Obviously, whether a person ever drinks coffee is a crucial variable for predicting if she will use a coupon for a coffee shop. Whether the person has an urgent place to go, whether the coffee shop is in the same direction as the destination, and whether the coupon is going to expire immediately are important variables for prediction too. The other variables seem to be of minimal importance.

\begin{table}[htp]
\begin{center}
\begin{tabular}{c|cc|cc}
				&  \multicolumn{2}{c|}{Least Reliance} &	\multicolumn{2}{c}{Least Error}\\
				& \multicolumn{2}{c|}{on noUrgentPlace} & \multicolumn{2}{c}{Logistic Model} \\
				& $\beta$		& $VI$ & $\beta$		& $VI$ \\
\hline
intercept			& 0.57		& 		& 0.07	& 		\\
zeroCoffee		& -2.18		& 1.27	& -2.03	& 1.26	\\
{\bf noUrgentPlace}	& {\bf 0.70}	& {\bf 1.06}	& {\bf 1.05}	& {\bf 1.10}	\\
sameDirection		& -1.30		& 1.05	& -0.93	& 1.04	\\
expOneDay		& 0.71		& 1.04	& 0.64	& 1.04	\\
withFriends		& 0.46		& 1.02	& 0.17	& 1.00	\\
male				& 0.49		& 1.01	& 0.18	& 1.00	\\
sunny			& 0.24		& 1.00	& 0.14	& 1.00	\\
				&  \multicolumn{2}{c|}{logistic loss = 2366} &	\multicolumn{2}{c}{logistic loss = 2296}\\	
\end{tabular}
\end{center}
\caption{Reliance and coefficient of the optimal model and the logistic regression estimator. This table shows that as model reliance on noUrgentPlace becomes small, model reliance for all other variables increases.}
\label{coupon_trade_off}
\end{table}

Suppose we think a priori that the variable noUrgentPlace is unreliable since the survey does not actually place people in an ``urgent'' situation. in that case, we may want to find an accurate predictive model with the least possible reliance on this variable. This is possible with VIC. Table \ref{coupon_trade_off} illustrates the trade-off.

The first and third columns in the table are the coefficient vectors for the two different models. The first column represents the model with the least reliance on noUrgentPlace within the VIC. The coefficients in the third column are for the plain logistic regression model. The second and fourth columns in the table are the model reliance vectors for the two model. The second column is the vector in VIC that minimizes the reliance on noUrgentPlace. The fourth column is the model reliance vector for the plain logistic regression model. By comparing the second and fourth column, we see that we can find an accurate model that relies on noUrgentPlace less. However, its logistic loss is 2366, which is about 3\% higher than the logistic regression model. This illustrates the trade-off between reliance and accuracy. By comparing these two columns, we also find that as we switch to a model with the least reliance on noUrgentPlace, the reliance on zeroCoffee, sameDirection and withFriends increases.

\subsection{Experiment 3: Image Classification}

The VIC analysis can be useful for any domain, including image classification and other problems that involve latent representations of data. We want to study how image classification relies on each of the latent features and how models with reasonable prediction accuracy can differ in terms of their reliance of these features.

We collected 1572 images of cats and dogs from ImageNet, and we use VGG16 features \cite{SZ2014} to analyze them. We use the convolutional base to extract features and train our own model. In particular, we get a vector of latent features of length 512 for each of our images. That is, our input dataset is $(\phi(\X), \y)$ of size 1572 $\times$ (512+1), where $\phi(\X)$ is the latent representation of the raw data $\X$. 

To get a sense of the performance of the pre-trained VGG16 model on our dataset as a benchmark, we build a fully connected neural network with two layers and train it with the data $(\phi(\X), \y)$. The accuracy of this model is about 75\% on the training set. We then apply logistic regression and perform the VIC analysis on the dataset. Given the large dimension of features and relatively small sample size, we impose an $l_1$ penalty on the loss function. With cross validation to select the best penalty parameter, we get a logistic model with non-zero coefficients on 61 of the features. The accuracy of this classifier on training sample is about $74\%$, which is approximately the same as the neural network we trained. We will restrict our analysis to the 61 features for simplicity.

We use the same method as Section 5.1 and randomly sample 417 logistic models in the Rashomon set. Moreover, we divide these models into 4 clusters. The idea is that similar models may have similar variable importance structure. We restrict our attention to the four latent features with the highest variable importance and construct the VID (see Figure \ref{VID_image}, where the colors represent the four clusters).

\begin{figure}
\centering
\includegraphics[scale = 0.35]{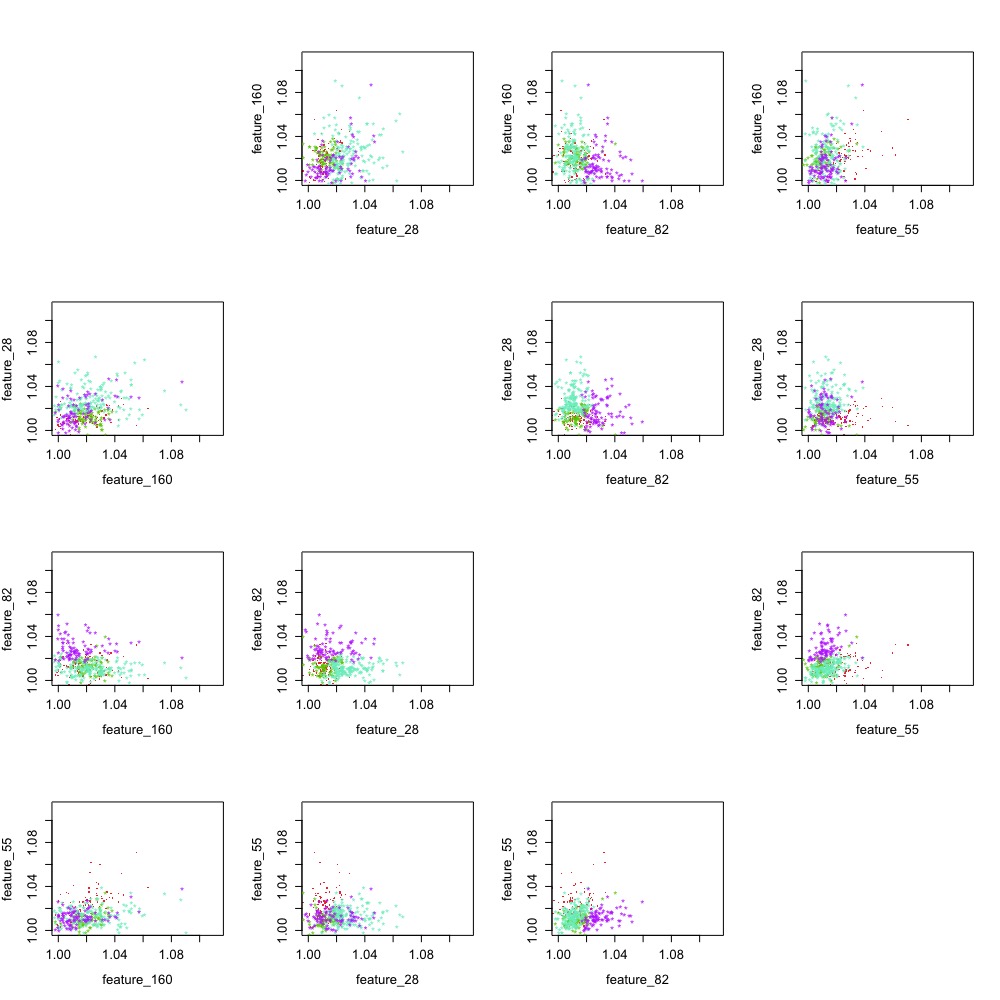}
\caption{The VID for Image Classification. The color corresponds the the clusters identified in the previous figure.}
\label{VID_image}
\end{figure}

From the VID, we gain a comprehensive view of the importance of these four latent features. From there, we would like to dig more deeply into the joint values of variable importance for models within the Rashomon set. For example, we do not know how a model that relies heavily on feature 160 and feature 28 is different from a model that does not rely on them at all. 

To answer this question, we select a representative model from each of the clusters and visualize these four models. We consider the following visualization method. Given an input image, we ask a model the counterfactual question: How would you modify the image so that you believe it is more likely to be a dog/cat? Given the functional form of the model, gradient ascent would answer this question. We choose a not-too-large step size and number of iterations so that the modified images of the four representative models are not too far from the original ones (so that we can interpret them) yet display significant differences (see Figure \ref{dog-and-cat}). 

\begin{figure}
\centering
\includegraphics[scale = 0.5]{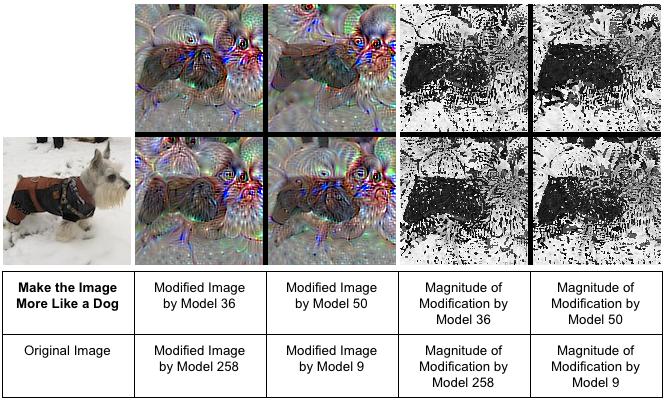}\\
\includegraphics[scale = 0.5]{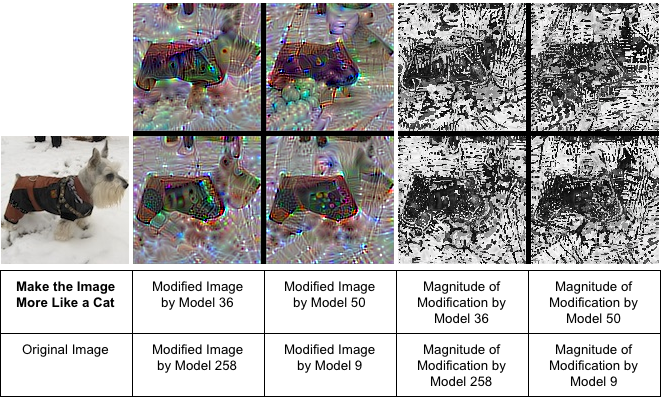}
\caption{Visualizing Representative Models. The left part of the figure is the original image. The middle part is the modified image by the four representative models, among which the upper (lower) half are the modified images with larger probability of being classified as a dog (cat). The right part of the figure are the gray-scale images that track the magnitudes of the modifications.}
\label{dog-and-cat}
\end{figure}

The upper panel represents the modification that increases the probability for being a dog, and the lower panel does the opposite. In each panel, the left part is the original image. The middle part is the output images after modification. The right part is the gray-scale images of the absolute value of the difference between the input and output images. In a gray-scale image, a darker pixel indicates larger modification. Note that the gray-scale images do not differentiate how the pixels are modified. For example, the models ``amplify'' the head to make it more like a dog, while they ``erase'' the head to make it more like a cat. The gray-scale images do not tell what operation (amplify or erase) is implemented on the image, for that we need to look at the images on the left.

Overall, the four representative models modify the input image similarly.  However, they are very different if we look at the details. In the upper panel, for example, we can see that Model 258 ``creates'' a dog head in the air above the body of the dog and Model 50 modifies this part of the input image similarly. The other two models create an eye above the body of the dog. The part of the image around the ear of the dog is another example. Model 36 does not modify much of this part, while the other models create another eye. The four models are also different when they modify the input image and make it more like a cat in the lower panel.

This experiment attempts to bridge the gap between the importance of latent features and the subtle differences among almost-equally accurate models for image classification. We believe that more work could be done along this direction to understand black-box algorithms for image classification.

 \section{Related Work}

As far as we know, there is no other work that aims to visualize the set of variable importance values for the \textit{set} of approximately-equally-good predictive models for a given problem. Instead, past work has mainly defined variable importance for \textit{single} predictive models, and our discussion in this section is aimed mostly centered around this body of work. The closest work to ours considers extreme statistics of the Rashomon set without characterizing it \citep{FRD2018,CRK2018}. 
While extreme statistics are useful to understand extreme cases, a full characterization of the set provides a much deeper understanding.

There are many variable importance techniques for considering single models. Breiman's variable importance measure for random forests \cite{Breiman2001}, which the VIC uses, as well as the partial dependence plot (PDP) \cite{Friedman2001}, partial leverage plot (PLP) \cite{VW1981}, and partial importance (PI) \cite{CMB2018} are four such examples. Some of these (e.g., the PDP) looks at the local importance of a variable while Breiman's variable importance measures global importance. The PDP measures importance by the difference in prediction outcomes while the VIC measures by the difference in prediction losses. The PLP is defined only for linear models, unlike the other variable importance measures.

The vast majority of work about variable importance is {\em posthoc}, meaning that it addresses a single model that has been chosen prior to the variable importance analysis. These works do not explore the class of models that could have been chosen, and are approximately equally good to the one that was chosen. 

Probably the most direct posthoc method to investigate variable importance is to simply look at the coefficients or weights of the model, after normalizing the features. For example, a zero coefficient of a variable indicates no importance, while a large coefficient indicates greater importance. This interpretation is common for linear models (\cite{Breimanetal2001}, \cite{GML2003}), and is also applicable to some of the non-linear models. This is a posthoc analysis of variable importance: It tells that a variable is important because the prediction is sensitive to the value of this variable, {\em if we select this predictive model}. Yet it does not posit that this variable is important to every good predictive model, and we could have selected another equally good predictive model in which this variable is not important at all.

In addition to looking at the coefficients or weights, there are many more sophisticated posthoc analyses of variable importance in various domains. Visual saliency \cite{HKP2007}, for instance, is not a measure of variable importance that has been been extended to an entire class of good models. Visual saliency tells us only what part of an image a single model is using. It does not show what part of that image every good model is choosing. However, it is possible to extend the VIC idea to visual saliency, where one would attempt to illustrate the range of saliency maps arising from the set of good predictive models.

There are several posthoc methods of visualizing variable importance. For linear models, the partial leverage plot \cite{VW1981} is a tool that visualizes the importance of a variable. To understand the importance of a variable, it extracts the information in this variable and the outcome that is not explained by the rest of the variables. The shape of the scatter plot of this extracted information informs us of the importance of the variable. The partial dependence plot \cite{Friedman2001} is another method that visualizes the impact of a variable on the average prediction. By looking at the steepness of the plot, one can tell the magnitude of the change of predicted outcome caused by a local change of a variable. One recent attempt to visualize variable importance is made by \cite{CMB2018}. They introduce a local variable importance measure and propose visualization tools to understand how changes in a feature affect model performance both on average and for individual data points. These methods, while useful, take a given predictive model as a primitive and visualize variable importance with respect to this single model. They neglect the existence of other almost-equally accurate models and the fact that variable importance can be different with respect to these models.

\section{Conclusion}

In this paper, we propose a new framework to analyze and visualize variable importance. We analyze this for linear models, and extend to non-linear problems including logistic regression and decision trees. This framework is useful if we want to study the importance of a variable in the context of the importance of other variables. It informs us, for example, how the importance of a variable changes when another variable becomes more important as we switch among a set of almost-equally-accurate models. We show connections from variable importance to hypothesis testing for linear models, and the trade-off between accuracy and model reliance.
\newpage

\bibliographystyle{informs2014}
\bibliography{reference}

\appendix
\section*{Appendices}

\section{Proof of Corollary \ref{scale}}\label{AppendixA}
\counterwithin{theorem}{section}
\setcounter{theorem}{0}

We first look at how the center of the Rashomon set is affected by the scale of data.
\begin{lemma}\label{A1}
Let $f_{\beta^*}$ be the linear model that minimizes the expected loss for $(X,Y)$ and $f_{\tilde{\beta}^*}$ for $(\tilde{X},\tilde{Y})$. If follows that $\tilde{\beta}^* = tS^{-1} \beta^*$ and $L(f_{\tilde{\beta}^*}; \tilde{X}, \tilde{Y}) = t^2 L(f_{\beta^*}, X, Y)$, where $S = diag(s_1, \cdots, s_p)$.
\end{lemma}

\begin{proof}
Since $\E[\tilde{X}\tilde{X}^T] = S \E[XX^T] S^T$ and $\E[\tilde{X}\tilde{Y}] = tS\E[XY]$, it follows that 
\begin{align*}
\tilde{\beta}^* 
= & \E[\tilde{X}\tilde{X}^T]^{-1}\E[\tilde{Y}\tilde{X}^T] \\
= & (S^{-T} \E[XX^T]^{-1} S^{-1}) (tS\E[XY])\\
= & tS^{-1}\E[XX^T]\E[XY] \\
= & tS^{-1}\beta^*
\end{align*}
\end{proof}

Lemma \ref{A1} shows that if the data $(X, Y)$ is scaled by $(S, t)$, then the center of the Rashomon set is scaled by $tS^{-1}$. We will show that in addition to the center, the whole Rashomon set is scaled by the same factor.

\begin{lemma}\label{A2}
For any $\beta \in \R^p$, let $\tilde{\beta} := tS^{-1} \beta$. Then $L(f_{\tilde{\beta}}; \tilde{X}, \tilde{Y}) = t^2 L(f_\beta; X, Y)$.
\end{lemma}
\begin{proof}
By definition,
\begin{align*}
L(f_{\tilde{\beta}}; \tilde{X}, \tilde{Y}) = & \tilde{\beta}^T \E[\tilde{X}\tilde{X}^T]\tilde{\beta} - 2\E[\tilde{Y}\tilde{X}^T]\tilde{\beta} + \E[\tilde{Y}^2] \\
= & (t\beta^T S^{-T}) S \E[XX^T] S^T (tS^{-1}\beta) \\
    & \hspace{1cm} - 2t\E[YX^T] S^T tS^{-1}\beta + t^2\E[Y^2] \\
= & t^2(\beta^T \E[XX^T]\beta - 2\E[YX^T]\beta + \E[Y^2]) \\
= & t^2 L(f_\beta; X, Y)
\end{align*}
\end{proof}

\begin{lemma}\label{A3}
If $f_\beta \in \Rash(X, Y)$, then $f_{\tilde{\beta}} \in \Rash(\tilde{X}, \tilde{Y})$, where $\tilde{\beta} = tS^{-1} \beta$.
\end{lemma}
\begin{proof}
We know from Lemma \ref{A2} $L(f_{\tilde{\beta}}; \tilde{X}, \tilde{Y}) = t^2 L(f_\beta, X,Y)$, and from Lemma $A1$ $L(f_{\tilde{\beta}^*}; \tilde{X}, \tilde{Y}) = t^2 L(f_\beta^*, X,Y)$
The fact that $f_\beta \in \Rash(X, Y)$ implies that $L(f_\beta; X,Y) \leq L(f_{\beta^*}; X,Y)(1+\epsilon)$. Multiply the inequality by $t^2$ would yield $L(f_{\tilde{\beta}}; \tilde{X}, \tilde{Y}) \leq L(f_{\tilde{\beta}^*}; \tilde{X}, \tilde{Y})(1+\epsilon)$.
\end{proof}

Once we know how the Rashomon set is scaled, we apply the model reliance function to see how the VIC is scaled.

\begin{proof}[Proof of Corollary \ref{scale}]
Recall that $\mr = \MRel(f_\beta; X, Y)$ with
\begin{align*}
\mr_j(f_\beta; X, Y) = & 2Cov(Y, X_j)\beta_j - 2 \beta_{-j}^T Cov(X_{-j}, X_j)\beta_j \\
= & 2Cov(Y, X_j)\beta_j - 2 \beta^T Cov(X, X_j)\beta_j + 2 Var(X_j)\beta_j^2,
\end{align*}
for $j = 1, \cdots, p.$
Then if a vector $\mr \in \VIC(X,Y)$, there exists a model $f_\beta \in \Rash(X,Y)$ by definition. By Lemma \ref{A3}, $f_{\tilde{\beta}} \in \Rash(\tilde{X}, \tilde{Y})$, where $\tilde{\beta} = tS^{-1}\beta$. It follows that
\begin{align*}
\widetilde{\mr}_j(f_{\tilde{\beta}};\tilde{X}, \tilde{Y}) = & 2Cov(\tilde{Y}, \tilde{X}_j)\tilde{\beta}_j - 2 \tilde{\beta}^T Cov(\tilde{X}, \tilde{X}_j)\tilde{\beta}_j + 2 Var(\tilde{X}_j)\tilde{\beta}_j \\
= & 2tCov(Y, X_j) s_j  (t s_j^{-1} \beta_j)  \\
  & - 2 (t\beta S^{-T}) S Cov(X, X_j) s_j (ts_j^{-1} \beta_j) \\
  & + 2 s_j^2 Var(X_j)(t^2s_j^{-2}\beta_j^2)\\
= & t^2 \mr_j(f_\beta; X,Y).
\end{align*}
That is, the vector $t^2 \mr \in \VIC(\tilde{X}, \tilde{Y})$.
\end{proof}

\section{Proof of Corollary \ref{VIC_lm_special_thm}}\label{AppendixB}
\counterwithin{theorem}{section}
\setcounter{theorem}{0}

\begin{proof}
For simplicity, let $\sigma_{ij} = \E(X_iX_j)$ and $\sigma_{iY} = \E(YX_i)$. Equation \ref{Rash_lm} can be written as\footnote{Simplify equation \ref{Rash_lm} with the fact that $\E(X_iX_j) = 0$ for all $i \neq j$ to get the first inequality below. By completing the squares, we get the second inequality, where the terms in the parentheses is the minimum loss.
\begin{align*}
\sum_{i=1}^p \left( (\sigma_{ii}+c)\beta_i^2 - 2\sigma_{iY}\beta_i \right) + \E(Y^2) & \leq \left(\E(Y^2) + \sum_{i=1}^p \frac{\sigma_{iY}^2}{\sigma_{ii}+c}\right)(1 + \epsilon) \nonumber \\
\sum_{i=1}^p (\sigma_{ii}+c)(\beta_i - \frac{\sigma_{iY}}{\sigma_{ii}+c})^2 & \leq \epsilon \left(\E(Y^2) + \sum_{i=1}^p \frac{\sigma_{iY}^2}{\sigma_{ii}+c}\right) \nonumber \\
\end{align*}
}
\begin{equation}\label{Rash_lm_special}
\sum_{i=1}^p \frac{(\beta_i - \frac{\sigma_{iY}}{\sigma_{ii}+c})^2}{\left(\sqrt{\frac{\epsilon L^*}{\sigma_{ii}+c}}\right)^2} \leq 1.
\end{equation}
Equation \ref{mr_lm} is also simplified as 
\begin{equation}\label{MR_lm_special}
\mr_j = 2\sigma_{jY}\beta_j.
\end{equation}
By plugging equation \ref{MR_lm_special} into \ref{Rash_lm_special}, we get the expression for VIC with uncorrelated features,\begin{align}\label{VIC_lm_special}
\sum_{i=1}^p \frac{(\mr_i - \frac{2\sigma_{iY}^2}{\sigma_{ii}+c})^2}{\left(2\sigma_{iY} \sqrt{\frac{\epsilon L^*}{\sigma_{ii}+c}}\right)^2} \leq 1,
\end{align}
This suggests that VIC with uncorrelated features is an ellipsoid with the center and axes specified in corollary \ref{VIC_lm_special_thm}. It follows that $r_i > r_j$ if and only if 
$$\frac{\abs{\sigma_{iY}}}{\sqrt{\sigma_{ii}+c}} > \frac{\abs{\sigma_{jY}}}{\sqrt{\sigma_{jj}+c}}.$$
By Corollary \ref{scale}, we can rescale the data $(X, Y)$ when $c = 0$. It follows that $r_i > r_j$ if and only if $$\abs{\rho_{iY}} > \abs{\rho_{jY}}.$$
\end{proof}

\section{Approximated VIC for Linear Models: }\label{AppendixC}

In Section \ref{VIC_approx_section}, we have approximated the VIC by an ellipsoid characterized by Equation \ref{VIC_lm_approx}. In particular, if we choose to invoke the Taylor approximation at the center of the Rashomon set, setting $\bar{\beta} = \beta^*$, we get the following equation by Theorem \ref{VIC_lm_approx_thm},
\begin{align*}
\widetilde{\mr}^T J^{-T} \E [XX^T + cI] J^{-1} \widetilde{\mr} - 2(\E[YX^T] - \beta^{*T}\E[XX^T+cI] )J^{-1}\widetilde{\mr} \leq \epsilon L^*,
\end{align*}
where $\widetilde{\mr} = \mr - \overline{\mr}$ and $L^* = L(f_{\beta^*})$.
Our purpose here is to understand the shape of the approximated VIC numerically. Since the approximated VIC is an ellipsoid, the key is to find its center, radii, and how it is rotated.

Let $A = J^{-T} \E [XX^T + cI] J^{-1} \in \R^{p*p}$. Note first that $A$ is symmetric and positive semi-definite. Therefore, it can be written as $A = Q \Lambda Q^T$ where $Q$ is an orthogonal matrix with each column being an eigenvector of $A$ and $\Lambda$ is a diagonal matrix that consists of the corresponding eigenvalues $\lambda_1, \cdots, \lambda_p$. {\em Assume $A$ is positive definite, } then $\lambda_j > 0$ for all $j$. It follows that 
$$\widetilde{\mr}^T Q \Lambda Q^T \widetilde{\mr} + 2(\beta^{*T}\E[XX^T] - \E[YX^T])J^{-1}\widetilde{\mr} \leq \epsilon L^*.$$
Let $\widehat{\mr} = Q^T \widetilde{\mr}$, and $B = (\E[YX^T] - \beta^{*T}\E[XX^T+cI] )J^{-1}Q= [b_1, \cdots, b_p] \in \R^p$. We have
$$\sum_{j=1}^p \left( \lambda_j \widehat{\mr}_j^2 - 2b_j\widehat{\mr}_j \right) \leq \epsilon L^*,$$ or equivalently,
$$\sum_{j=1}^p \lambda_j(\widehat{\mr}_j - \frac{b_j}{\lambda_j})^2 \leq \epsilon L^* + \sum_{j=1}^p \frac{b_j^2}{\lambda_j}.$$
\begin{equation}\label{VIC_approx_ellipsoid}
\sum_{j=1}^p \frac{(\widehat{\mr}_j - \frac{b_j}{\lambda_j})^2}{\left(\sqrt{\frac{\epsilon L^* + \sum_{j=1}^p b_j^2 / \lambda_j}{\lambda_j}}\right)^2} \leq 1.
\end{equation}
The expression \ref{VIC_approx_ellipsoid} implies that the approximated VIC is an ellipsoid. Unlike the VIC for uncorrelated features, the ellipsoid no longer parallels the coordinate axes. The eigenvectors of the matrix $J^{-T} \E [XX^T +cI] J^{-1}$ determine how the VIC ellipse is rotated. The eigenvalues together with $b_j$'s determine the center and radii of the ellipsoid.

\section{Proof of Theorem \ref{hypothesis}}\label{AppendixD}

Consider the dataset $\{(x_i^T, y_i)\}_{i=1}^{n}$, where $x_i^T$ is the the features of observation $i$, $y_i$ is the outcome. We assume the following:

\begin{assumption}[Linear specification]
For all $i = 1, \cdots, n$, $y_i = x_i^T \alpha + \epsilon_i$, where $\alpha$ is the true coefficient and $\epsilon_i$'s are the error terms.
\end{assumption}

\begin{assumption}[IID]
$\{x_i^T, \epsilon_i\}_{i=1}^n$ are i.i.d. 
\end{assumption}

\begin{assumption}[Exogeneity]
For all $i = 1, \cdots, n$, $\E(g_i) = 0$, where $g_i := x_i \epsilon_i$.
\end{assumption}

\begin{assumption}[Rank condition]
For all $i = 1, \cdots, n$, $\Sigma_{xx} := \E(x_i x_i^T)$ is non-singular.
\end{assumption}

\begin{assumption}[Finite second moment]
For all $i = 1, \cdots, n$, $S := \E(g_i g_i^T)$ is finite.
\end{assumption}

\begin{assumption}[Consistent estimator of $S$]
There is an estimator $\hat{S}$ with $\hat{S} \overset{p}{\longrightarrow} S$.
\end{assumption}

\begin{assumption}[Non-singularity of $S$]
$S$ is non-singular.
\end{assumption}

Since the purpose is not to prove the asymptotic properties of the least squares estimator, we assume directly the existence of $\hat{S}$ instead of deriving it. We begin by a standard result.

\begin{lemma}\label{step1}
The least squares estimator $\hat{\beta}$ satisfies $$\sqrt{n}(\hat{\beta} - \alpha) \overset{d}{\longrightarrow} \mathcal{N}(0, \Var(\hat{\beta})),$$
where $\Var(\hat{\beta}) = \Sigma_{xx}^{-1} S \Sigma_{xx}^{-1}$.
\end{lemma}

This is a property of the least squares estimator. \footnote{The proofs in Appendix \ref{AppendixD} are standard and can be found in many textbooks. (See, for example, \cite{Hayashi2000}.) Hence we omit the proofs in this appendix.}

Suppose we are interest in the model reliance of the $j^{th}$ feature. For the linear model $f_\beta$, the reliance is given by the function below: $$\MRel_j(\beta) = 2\Cov(Y, X_j)\beta_j - 2 \beta^T \Cov(X, X_j)\beta_j + 2 \Var(X_j)\beta_j^2.$$
Notice that this function relies on the population distribution.

By Delta Method and Lemma \ref{step1}, we have 
\begin{equation}\label{D1}
\sqrt{n}(\MRel_j(\hat{\beta}) - \MRel_j(\alpha))\overset{d}{\longrightarrow} \mathcal{N}(0, \Sigma),
\end{equation}
where $$\Sigma = \nabla^T \MRel_j(\alpha) \Var(\hat{\beta})\nabla \MRel_j(\alpha)$$

\begin{lemma}\label{step2}
$\Sigma$ is positive definite. 
\end{lemma}

One can prove this by checking the definitions. As a result, Equation \ref{D1} implies the following. 
\begin{equation}\label{D2}
\Sigma^{-1/2}\sqrt{n}(\MRel_j(\hat{\beta}) - \MRel_j(\alpha))\overset{d}{\longrightarrow} \mathcal{N}(0, I).
\end{equation}

Define the empirical analogous of $\Sigma$ as 
$$\hat{\Sigma} := \nabla^T \widehat{\MRel}_j(\hat{\beta}) \widehat{\Var}(\hat{\beta})\nabla \widehat{\MRel}_j(\hat{\beta}),$$
where $\widehat{\Var}(\hat{\beta}) = \Sigma_{xx}^{-1} \hat{S} \Sigma_{xx}^{-1}$ and the $\widehat{\MRel}_j$ function is defined by replacing the population variance and covariance by the sample analogs.

\begin{lemma}\label{step3}
$\hat{\Sigma}\overset{p}{\longrightarrow}\Sigma$ and $\hat{\Sigma}$ is positive definite.
\end{lemma}

Given this result, the Continuous Mapping Theorem implies that 
$$\hat{\Sigma}^{-1/2}\overset{p}{\longrightarrow}\Sigma^{-1/2}.$$
Combined with Equation \ref{D2} and by applying Slutzky's Theorem, it follows that 
\begin{equation}\label{D3}
\hat{\Sigma}^{-1/2}\sqrt{n}(\MRel_j(\hat{\beta}) - \MRel_j(\alpha))\overset{d}{\longrightarrow} \mathcal{N}(0, I).
\end{equation}

Since $\widehat{\MRel}_j(\hat{\beta})\overset{p}{\longrightarrow}MR_j(\hat{\beta})$. It follows from Equation \ref{D3} that
$$\hat{\Sigma}^{-1/2}\sqrt{n}(\widehat{\MRel}_j(\hat{\beta}) - \MRel_j(\alpha))\overset{d}{\longrightarrow} \mathcal{N}(0, I).$$
To complete the proof:
\begin{align*}
 & n(\widehat{\MRel}_j(\hat{\beta}) - \MRel_j(\alpha))^T \hat{\Sigma}^{-1}(\widehat{\MRel}_j(\hat{\beta}) - \MRel_j(\alpha)) \\
= & \sqrt{n}(\widehat{\MRel}_j(\hat{\beta}) - \MRel_j(\alpha))^T (\hat{\Sigma}^{-T/2}\hat{\Sigma}^{-1/2}) \sqrt{n}(\widehat{\MRel}_j(\hat{\beta}) - \MRel_j(\alpha)) \\
= & \left[\hat{\Sigma}^{-1/2}\sqrt{n}(\widehat{\MRel}_j(\hat{\beta}) - \MRel_j(\alpha))\right]^T \left[\hat{\Sigma}^{-1/2}\sqrt{n}(\widehat{\MRel}_j(\hat{\beta}) - \MRel_j(\alpha))\right] \\
\overset{d}{\longrightarrow} & \chi_1^2.
\end{align*}

\end{document}